\documentclass[sigconf]{acmart}

\usepackage[capitalize]{cleveref}
\usepackage{multirow}
\usepackage{array}
\usepackage{booktabs}
\usepackage{mathtools}
\usepackage{algorithm}
\usepackage[noend]{algorithmic}
\usepackage{subcaption}

\renewcommand\footnotetextcopyrightpermission[1]{} 

\newcommand{\xhdr}[1]{\vspace{1.0mm}\noindent{{\bf #1.}}\hspace{0.5mm}}
\newcommand{\R}{\mathbb{R}}

\DeclareMathOperator{\E}{\mathrm{E}}
\DeclareMathOperator{\Span}{span}
\DeclareMathOperator{\Vector}{vec}

\theoremstyle{acmdefinition}
\newtheorem{excont}{Example}

\begin{document}

\title{Learning Interpretable Feature Context Effects \\ in Discrete Choice}

\author{Kiran Tomlinson}
\email{kt@cs.cornell.edu}
\affiliation{%
  \institution{Cornell University}
}

\author{Austin R. Benson}
\email{arb@cs.cornell.edu}
\affiliation{%
  \institution{Cornell University}
}

\begin{abstract}
  The outcomes of elections, product sales, and the structure of social connections are all determined by the choices individuals make when presented with a set of options, so understanding the factors that contribute to choice is crucial. Of particular interest are context effects, which occur when the set of available options influences a chooser's relative preferences, as they violate traditional rationality assumptions yet are widespread in practice. 
  However, identifying these effects from observed choices is challenging, often requiring foreknowledge of the effect to be measured. 
  In contrast, we provide a method for the automatic discovery of a broad class of context effects from observed choice data. 
Our models are easier to train and more flexible than existing
models and also yield intuitive, interpretable, and statistically testable context effects.
Using our models, we identify new context effects in widely used choice datasets and provide the first analysis of choice set context effects in social network growth. 
\end{abstract}





\maketitle

\pagestyle{plain} 

\section{Introduction}

Understanding human choice is a central task in behavioral psychology, economics, and neuroscience 
that has garnered interest in the machine learning community due to recent increases in automated data collection and the power of data-driven modeling~\cite{maystre2015fast,benson2018discrete,seshadri2019discovering,rosenfeld2020predicting,tomlinson2020choice}. 
In a \emph{discrete choice} setting, an individual chooses between a finite set of available items called a \emph{choice set}. 
This general framework describes a host of important scenarios, including purchasing~\cite{lockshin2006using}, transportation decisions~\cite{ben2018discrete}, voting~\cite{dow2004multinomial}, and the formation of new social connections~\cite{overgoor2019choosing,gupta2020mixed}.
Discovering and understanding the factors that contribute to the choices people make has broad applications in, e.g.,
forecasting future choices~\cite{pathak2020well}, product or policy design~\cite{brownstone1996transactions}, and
recommender systems~\cite{yang2011collaborative,ruiz2020shopper}.

In the simplest model of choice, we might hypothesize that each available option has some intrinsic value (or \emph{utility}) to the chooser, who selects the item with maximum utility. However, human choices are not deterministic: when presented with the same menu on two different visits to a restaurant, people often place different orders. We could therefore make the model probabilistic, and say that the probability an item is chosen is proportional to its utility; this is exactly the Plackett-Luce model \cite{luce1959individual,plackett1975analysis}. Alternatively, we could assume that individuals observe random utilities for each item before selecting the maximum utility item (from a psychological perspective, the mechanism behind random utilities could be imperfect access to internal desires or actual stochastic variation in preferences). These are known as random utility models (RUMs) \cite{marschak1960rum}, the most famous of which is the multinomial (or conditional) logit (MNL) \cite{mcfadden1973conditional}.

The Plackett-Luce and MNL models both obey the axiom of \emph{independence of irrelevant alternatives} (IIA) \cite{luce1959individual}, that relative preferences between items are unaffected by the choice set --- if someone prefers $x$ to $y$, they should still do so when $z$ is also an option. However, experiments on human decision-making~\cite{huber1982adding,simonson1992choice,shafir1993reason,trueblood2013context}
as well as direct measurement on choice data~\cite{mcfadden1977application,small1985multinomial,benson2016relevance,seshadri2019discovering}
have found that this assumption often does not hold in practice
(note, though, that tests on choice data sometimes suffer from computational complexity issues~\cite{seshadri2019fundamental}).
These ``IIA violations'' are termed \emph{context effects}~\cite{prelec1997role,rooderkerk2011incorporating}. 
Examples of observed context effects include the \emph{asymmetric dominance} (or \emph{attraction}) \emph{effect} \cite{huber1982adding}, where the presence of an inferior item makes its better equivalent more desirable; the \emph{compromise effect} \cite{simonson1989choice}, where people tend to prefer middle-ground options; and the \emph{similarity effect} \cite{tversky1972elimination}, where similar items split the preferences of the chooser (for instance, in voting, this is known as the \emph{spoiler effect} or \emph{vote splitting}).
 
The ubiquity of context effects in human choice has driven the development of more nuanced models capable of capturing these effects. In machine learning, the goal is typically to design models that perform better in prediction tasks by taking advantage of learned context effects \cite{chen2016modeling,chen2016predicting,seshadri2019discovering,pfannschmidt2019learning,rosenfeld2020predicting,bower2020salient,ruiz2020shopper}. However, the effects accounted for by models incorporating neural networks and learned item embeddings \cite{pfannschmidt2019learning,chen2016predicting,rosenfeld2020predicting} are difficult to interpret. Other models learn context effects at the level of individual items~\cite{chen2016modeling,seshadri2019discovering,natenzon2019random,ruiz2020shopper}, preventing them from generalizing to items not appearing in the training set and making it difficult to discover context effects coming from item \emph{features} (e.g., price). Yet another approach involves hand-crafted nonlinear features specific to a dataset~\cite{bruch2016extracting,kim2007capturing}, which can work well for the problem at hand but does not provide a general methodology for modeling choice.

Within behavioral economics, context effect models tend to be engineered to describe very specific effects and are often only applied (if at all) to carefully controlled special-purpose datasets~\cite{kivetz2004alternative,rooderkerk2011incorporating,tversky1993context,leong2012embedding,bordalo2012salience,bordalo2013salience,masatlioglu2012revealed}. 
Psychological research on context effects in choice largely follows the same pattern while also introducing complex behavioral processes (for instance, time-varying attention) that are typically not estimable from general choice datasets~\cite{usher2004loss,roe2001multialternative,turner2018competing,howes2016contextual,tversky1972elimination}. 
Of course, special-purpose approaches have their use in rigorously establishing the existence of specific choice behaviors or exploring possible mechanisms of human psychology. However, leveraging the wealth of general choice datasets to discover new context effects requires a highly flexible and computationally tractable approach.

\xhdr{The present work: learning feature context effects}
Here, we provide methods for the automatic discovery (\emph{learning}) of a wide class of context effects from large, pre-existing, and disparate choice datasets. 
The key advantage of our approach over the previous work discussed above is that we can take a choice dataset collected in any domain (possibly one that has already been collected passively), 
efficiently train a model, and directly interpret the learned parameters as intuitive context effects. 
For example, we find in a hotel booking dataset that users presented with more hotels on sale showed increased willingness to pay; this allows us to hypothesize that observing many ``on sale'' tags exerts a context effect on the user, making them feel better about selecting a more expensive option. 
Context effects that our method extracts from choice data could then motivate further experimental work (e.g., A/B testing). We focus on the case where items are described by a set of features (e.g., for hotels: price, star rating, promotion status) and where the utility of each item is a function of its features. 
This setup has two major benefits, as it allows us to
(i) make predictions about new items not observed in the data and 
(ii) learn interpretable effects that provide generally applicable insight into human choice behavior.

We define \emph{feature context effects}, which describe the change in the importance of a feature in determining choice as a function of features of the choice set. 
For instance, suppose a diner is presented with two choice sets on different occasions, one consisting of fast food chains and the other of high-end restaurants. In the choice set with higher mean price, the diner is likely to place more weight on wine selection, while in the choice set with lower mean price, the diner could place more weight on service speed. 
We introduce two models --- the linear context logit (LCL) and decomposed linear context logit (DLCL) --- to learn these types of interpretable feature context effects from choice data. 

The linear context logit accounts for context effects by adjusting the importance of each feature according to the mean features over the choice set (this form naturally arises from simple assumptions on context effects). As the name indicates, we assume these relationships are linear for interpretability and ease of inference. In structure, the LCL is based on the multinomial logit, and it inherits all of its benefits --- in particular, a log-concave likelihood function, ease of identification, and compatibility with utility-maximization through the RUM framework. The LCL can thus be efficiently, optimally, and uniquely estimated from choice data using standard convex optimization methods. 

The decomposed linear context logit is more expressive, but slightly harder to estimate.
In the DLCL, we break up the context effects exerted by each feature into their own sub-models and weight these according to the relative strengths of the effects. In addition to more flexibly accounting for context, the DLCL also captures mixed populations where the preferences of different latent sub-populations are dominated by different effects. 

We perform an extensive analysis of choice datasets using our models, showing that statistically significant feature context effects occur and recovering intuitive effects. For example, we find evidence that people pick more expensive hotels when their choice sets have high star ratings, that people offered more oily sushi show more aversion to oiliness, and that when deciding whose Facebook wall to post on, people care more about mutual connections when choosing from popular friends.\footnote{By necessity, these are all correlative rather than causal claims; we discuss this more in the results section.} We show that accounting for feature context effects provides a boost in predictive power, although our primary focus is on learning interpretable context effects 
(there are more general, neural-network-based methods that can possibly improve prediction~\cite{rosenfeld2020predicting,pfannschmidt2019learning}). 
Additionally, we demonstrate how individual effects can be statistically tested and how sparsity-encouraging $L_1$ regularization can help identify the most influential context effects in our models. 

Our empirical study is split into two sections. First, we examine datasets specifically collected to understand preferences, covering a variety of choice domains including sushi, hotel bookings, and cars. Next, we apply our methods to social network analysis, where we demonstrate context effects in competing theories of triadic closure, which describes the tendency of new friendships to form among friends-of-friends~\cite{easley2010networks,granovetter1977strength}. Discrete choice models have recently found compelling use in studying social network dynamics~\cite{overgoor2019choosing,overgoor2020scaling,gupta2020mixed,feinberg2020choices}. Here, we show how incorporating context effects using our models can yield new sociological insights.

\xhdr{Additional related work}
The idea of ``context-dependent preferences'' is an old one; in a paper with that very title \cite{tversky1993context}, \citeauthor{tversky1993context} introduced the componential context model to account for many observed context effects, including the asymmetric dominance and compromise effects. The LCL model that we develop could be viewed as a RUM adaptation of the componential context model where the background contexts are the mean choice set features and the relative advantage term is omitted. Importantly, the LCL can be learned from choice observations, whereas the componential context model provides a possible description of behavior with no associated estimation method.

In the machine learning literature, the LCL is most similar in spirit to the context-dependent random utility model (CDM)~\cite{seshadri2019discovering} in that we consider pairwise contextual interactions, but with the important distinction that our model operates on features rather than on items, allowing for the discovery of general, non-item-specific effects. Our framework for context-dependent utilities is related to a special case of set-dependent weights~\cite{rosenfeld2020predicting} and FETA~\cite{pfannschmidt2019learning}, although those methods use neural network implementations. Other machine learning models for context effects include the blade-chest model~\cite{chen2016modeling} and its extensions~\cite{chen2016predicting}, which are restricted to pairwise comparisons, and the salient features model~\cite{bower2020salient}, which considers different subsets of features in each choice set.

Recent research has framed network growth (the formation of new connections in, e.g., communication or friendship networks) as discrete choice~\cite{overgoor2019choosing}. In that paper, \citeauthor{overgoor2019choosing} suggested context effects in network growth as a direction for more flexible modeling. 
The models we introduce are a first step in this direction, and we show that modeling context effects is useful for both improved predictive capability and generating additional insight into social processes. More recently still, some of the same authors introduced the de-mixed mixed logit~\cite{overgoor2020scaling} that allows varying preferences over disjoint choice sets (e.g., friends, friends-of-friends, and unrelated nodes). While this approach does allow for some IIA violations, it does not reveal whether (and if so, how) the features of items in each choice set affect the preferences of choosers. The same is true of recent work on applying mixed logit to network growth~\cite{gupta2020mixed}.

\section{Discrete Choice Background}\label{sec:discrete_choice}
We first briefly review the discrete choice modeling framework (see~\cite{train2009discrete,ben2018discrete} for thorough treatments), following the notation from~\cite{overgoor2019choosing}. In a discrete choice setting, an individual selects an item from a set of available items, the choice set. We use $\mathcal X$ to denote the universe of all items and $C \subseteq \mathcal X$ the choice set in a particular choice instance. A choice dataset $\mathcal D$ is a set of $n$ pairs $(i, C)$, where $i \in C$ is the item selected. Each item $i$ is described by vector of $d$ features $x_i \in \R^d$ that determine the preferences of the chooser.

Of particular economic interest are random utility models \cite{marschak1960rum} (RUMs), which are based on the idea that individuals try to maximize their utility, but can only do so noisily. In a RUM, an individual draws a random utility for each item (where each item has its own utility distribution) and selects the item with maximum observed utility. 
The workhorse RUM is the conditional multinomial logit (MNL)~\cite{mcfadden1973conditional}, which has interpretable parameters that are readily estimated from data.
In the MNL model, the observed utility of each item $i$ is the random quantity $\theta^T x_i + \varepsilon$, where the fixed but latent parameter $\theta \in \R^d$ (the \emph{preference vector}) stores the relative importance of each feature (the \emph{preference coefficients}) and the random noise term $\varepsilon$ follows a standard Gumbel distribution with CDF $e^{-e^{-x}}$. 
This noise distribution is chosen so that the MNL choice probabilities have a simple closed form \cite{train2009discrete}, namely, a softmax over the utilities. Under an MNL, the probability that $i$ is chosen from the choice set $C$, denoted $\Pr(i, C)$, is
\begin{equation}
	\Pr(i, C) = \frac{\exp(\theta^T x_i)}{\sum_{j \in C}\exp(\theta^T x_j)}.
\end{equation}

The MNL model famously obeys the axiom of \emph{independence of irrelevant alternatives} (IIA) \cite{luce1959individual}, stating that relative choice probabilities are unaffected by the choice set. Formally, a model satisfies IIA if for any two choice sets $C, C'$ and items $i, j \in C\cap C'$, 
\begin{equation}
 	\frac{\Pr(i, C)}{\Pr(j, C)} = \frac{\Pr(i, C')}{\Pr(j, C')}.
 \end{equation}

As we have discussed, this assumption is often violated in practice by context effects. One model that can account for context effects is the mixed logit (in fact, continuous mixed logit is powerful enough to encompass all RUMs~\cite{mcfadden2000mixed}). The DLCL model that we will introduce is related to a discrete mixed logit, so we briefly describe it here. In a discrete mixed logit, we have $M$ populations, each of which has its own preference vector $\theta_m$. The mixing parameters $\pi_1, \dots, \pi_M$ (with $\sum_{m=1}^M \pi_m = 1$) describe the relative sizes of the populations. This results in choice probabilities
\begin{equation}
	\Pr(i, C) = \sum_{m=1}^M \pi_m \frac{\exp(\theta_m^T x_i)}{\sum_{j \in C}\exp(\theta_m^T x_j)}.
\end{equation}
While mixed logit can produce IIA violations, it does so by hypothesizing populations each with their own context-effect-free preferences (context effects only appear in the aggregate data). In contrast, our models identify context effects in individual preferences. 

\section{Models of Feature Context Effects}\label{sec:models}

 \begin{figure}[t]
\centering
\includegraphics[width=\linewidth]{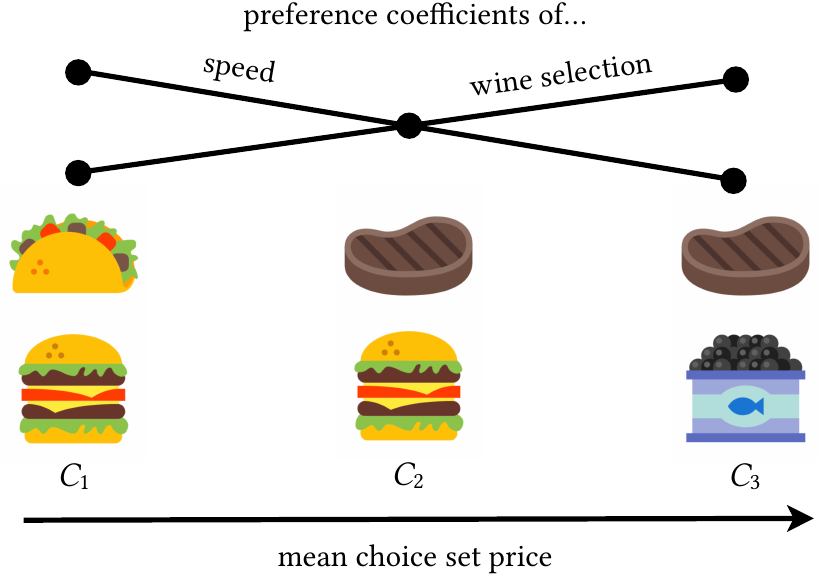}  	
\caption{Feature context effects in the toy restaurant example (\Cref{ex:quick_taco}). In $C_1 = \{$Quick Taco, Burger Express$\}$, more emphasis is placed on service speed and less on wine selection. In $C_2 = \{$Steak Deluxe, Burger Express$\}$, equal weight is placed on both. In $C_3=\{$Steak Deluxe, Ch\^{a}teau Caviar$\}$, more importance is placed on wine selection and less on service speed. We can view this as a positive effect of choice set price on the wine preference coefficient and a negative effect of choice set price on the speed preference coefficient. (Icons by \cite{icons8}.)
}
\label{fig:quick_taco}
\end{figure}

In order to capture context effects at the individual level, the choice set itself needs to influence the preferences of a chooser. 
In the most general extension of the MNL, we could replace $\theta$ with $\theta + F(C)$, where $F$ is an arbitrary function of the choice set (this is analogous to the set-dependent weights model \cite{rosenfeld2020predicting}, but framed as a RUM). This allows each feature to exert an arbitrary influence on the base preference coefficient of each other feature. 
We say that a \emph{feature context effect} occurs when $F(C) \ne 0$. We provide the following simple example to illustrate this idea.

\begin{example}\label{ex:quick_taco}
	Suppose we are choosing between restaurants, each described by three features: \emph{price}, \emph{service speed}, and \emph{wine selection} (encoded in that order and centered to have zero mean). The universe contains four restaurants: (1) Burger Express, (2) Quick Taco, (3) Steak Deluxe, and (4) Ch\^{a}teau Caviar, described by the following feature vectors:
\[
	x_1 = \begin{bmatrix}
	-1\\
	2\\
	-1
	\end{bmatrix},
	\quad
	x_2 = \begin{bmatrix}
	-1\\
	1\\
	-2
	\end{bmatrix},
	\quad
	x_3 = \begin{bmatrix}
	1\\
	0\\
	1
	\end{bmatrix},
	\quad
	x_4 = \begin{bmatrix}
	1\\
	-3\\
	2
	\end{bmatrix}.
\]
Suppose that under an MNL, our preference for low price, fast service speed, and good wine selection is encoded in a fixed preference vector $\theta = [-1, 1, 1]$. 
However, given the options $C_1=\{$Burger Express, Quick Taco$\}$, we might care more about speed and less about wine selection, modifying our preferences to $\theta+ F(C_1) = [-1, 2, 0]$. 
On the other hand, given the options $C_3=\{$Steak Deluxe, Ch\^{a}teau Caviar$\}$, we place more emphasis on wine selection and less on speed, so that $\theta + F(C_3) = [-1, 0, 2]$. 
In an intermediate choice set $C_2 = \{$Steak Deluxe, Burger Express$\}$, our preferences remain $\theta + F(C_2) = [-1, 1, 1]$ (i.e., $F(C_2) = 0$ and there is no net feature context effect). 
One natural $F$ that accounts for these context-dependent preferences is a function that increases the coefficient for speed when the prices in the choice set are low and increases the coefficient for wine selection when prices in the choice set are high. \Cref{fig:quick_taco} illustrates this idea.
\end{example}

We make two simplifying assumptions on the choice set effect function $F(C)$.
The first is that the effect of a choice set additively decomposes into effects of its items, i.e.,
$F(C)$ is proportional to $\sum_{j \in C} f(x_j)$ for some fixed function $f$. (The same assumption is made by the CDM \cite{seshadri2019discovering} and FETA~\cite{pfannschmidt2019learning}.)
While in principle higher-order interactions are possible (e.g., the effect of item $i$ on item $j$ depends on whether item $k$ is also present), the number of higher-order interactions is exponential in the size of the choice set; these interactions would not be visible in choice datasets (which typically do not contain observations from every possible choice set) and it is likely that higher-order effects are sparse~\cite{batsell1985new}.
Second, we assume that the effect of each item is diluted in large choice sets
and we model this with a proportionality constant of $1 / \lvert C \rvert$, i.e.,
$F(C) = 1/\lvert C \rvert\sum_{j \in C}f(x_j)$.

\subsection{Linear Context Logit}
While in principle features could exert arbitrary influences on each other, we focus on the case when context effects are linear, which makes inference tractable and, crucially, preserves interpretability. We use $x_{C} = 1/\lvert C \rvert \sum_{j\in C} x_j$ to denote the mean feature vector of the choice set $C$. 
For $f$ linear, we can write $f(x_j) = Ax_j$ for some matrix $A \in \R^{d\times d}$, and the choice set context function $F$ is
\[
F(C) = \frac{1}{\lvert C \rvert}\sum_{j \in C}f(x_j) = \frac{1}{\lvert C \rvert}\sum_{j \in C}Ax_j = Ax_C.
\]
We call this model the linear context logit (LCL), and it produces choice probabilities
\begin{equation}\label{eq:lcl}
	\Pr(i, C) = \frac{\exp(\left[\theta + Ax_C\right]^T x_i)}{\sum_{j \in C}\exp(\left[\theta + Ax_C\right]^T x_j)}.
\end{equation}
The model has $d^2+d$ parameters: $d^2$ for $A$ and $d$ for $\theta$. In the LCL, $A_{pq}$ specifies the effect of feature $q$ on the coefficient of feature $p$. 
If $A_{pq}$ is positive (resp.~negative), then higher values of $q$ in the choice set result in a higher (resp.~lower) preference coefficient for $p$. If $A = 0$, then the LCL degenerates to an MNL. 

When analyzing data in \cref{sec:results}, we often see strong effects from the diagonal entries of $A$.
The signs of the diagonal entries of $A$ can be explained by known context effects.
The case of $A_{pp} < 0$ is consistent with the similarity effect (similar options split the preferences of the chooser),
and the case of $A_{pp} > 0$ is consistent asymmetric dominance (similar options help reveal the inferior alternatives).

To further aid intuition, we consider how to express the context effects in the toy restaurant example with the LCL.

\begin{excont}[continued]
	We can encode the function $f$ in the LCL with the context effect matrix
	\[
	A = \begin{bmatrix}
		0 & 0 & 0\\
		-1 & 0 & 0 \\
		1 & 0 & 0
	\end{bmatrix}.
	\]
Recall that the order of features is price, service speed, and wine selection.
The two non-zero entries of $A$ can be directly interpreted as feature context effects: $A_{21}=-1$ says that when the choice set prices are higher, the importance of service speed decreases and $A_{31}=1$ says that when the choice set prices are higher, the importance of wine selection increases. Conversely, when prices are lower, service speed is more important and wine selection is less important.

We can now verify that this context effect matrix $A$ accounts for our varying preferences. In the choice sets $C_1, C_2,$ and $C_3$, the mean feature vectors are
	\[
	x_{C_1} = \begin{bmatrix}
		-1\\
		1.5\\
		-1.5
	\end{bmatrix},\quad 
	x_{C_2} = \begin{bmatrix}
		0\\
		1\\
		0
	\end{bmatrix},\quad
	x_{C_3} = \begin{bmatrix}
		1\\
		-1.5\\
		1.5
	\end{bmatrix}.
	\]
	Notice that the mean price (the first feature) steadily increases from $C_1$ to $C_2$ to $C_3$. According to the LCL (with base utilities $\theta = [-1, 1, 1]$), the context-adjusted preference vectors are $\theta + Ax_{C_1} = [-1, 2, 0]$, $\theta + Ax_{C_2}=[-1, 1, 1]$, and $\theta + Ax_{C_3}=[-1, 0, 2]$, as desired.
\end{excont}

Just as in the MNL, we can derive the closed form in \eqref{eq:lcl} if choosers have random utilities $\left[\theta + Ax_C\right]^T x_i + \varepsilon$, where $\varepsilon$ follows a standard Gumbel distribution and the random variable samples are i.i.d. If we want a more parsimonious model, we can impose sparsity on $A$ through $L_1$ regularization (we do this in our empirical analysis) or we could take a low constant-rank approximation of $A$ if the number of features is prohibitively large, making the number of parameters linear in $d$.

\subsection{Decomposed Linear Context Logit}
The LCL accounts for context effects of each feature on every other feature all at once. 
While this makes the LCL simple and parameter learning tractable (\cref{sec:estimation}), this part of the model can sometimes be limiting. 
More specifically, the LCL implicitly assumes that the intercepts of all linear context effects exerted by one feature are the same (we have $d^2$ slopes in $A$, but only $d$ intercepts in $\theta$). In other words, it assumes that the coefficient of feature $p$ is the same when feature $q = 0$ as when feature $q' = 0$. Since this is a subtle point, we provide an example of a linear context effect not expressible with an LCL.

\begin{example}\label{ex:dlcl}
	Suppose we have four choice sets with the following mean features:
	\[
		x_{C_1}=\begin{bmatrix}
		1\\
		0\\
		0
		\end{bmatrix},\quad 
		x_{C_2}=\begin{bmatrix}
		2\\
		0\\
		0
		\end{bmatrix},\quad 
		x_{C_3}=\begin{bmatrix}
		0\\
		1\\
		0
		\end{bmatrix},\quad 
		x_{C_4}=\begin{bmatrix}
		0\\
		2\\
		0
		\end{bmatrix}.
	\]
	We will assume that each choice set is large enough that the LCL is uniquely identifiable (see \Cref{prop:lcl_identify_simple} for a sufficient identifiability condition).
	Under the LCL, the coefficient of the third feature in these choice sets is:
		\[
		\begin{tabular}{ll}
			\toprule
			choice set & 3rd feature coefficient\\
			\midrule
			$C_1$ & $\theta_{3} + A_{31}$\\
			$C_2$ & $\theta_{3} + 2A_{31}$\\
			$C_3$ & $\theta_{3} + A_{32}$\\
			$C_4$ & $\theta_{3} + 2A_{32}$\\
			\bottomrule
		\end{tabular}
		\]
	We can determine $\theta_3$ and $A_{31}$ from observing the coefficient of the third feature in choice sets $C_1$ and $C_2$ (two equations, two unknowns). If we then observe the preference coefficient of the third feature in $C_3$, this allows us to find $A_{32}$ and thus determines the coefficient of the third feature in $C_4$. To get a context effect not expressible with the LCL, we can use $C_1, C_2$, and $C_3$ to uniquely specify $\theta_3, A_{31}$, and $A_{32}$ and then construct choice probabilities in $C_4$ so that the coefficient of the third feature is anything other than $\theta_{3} + 2A_{32}$. This example would be expressible if we allowed the linear context effects exerted by each feature to have varying intercepts as well as slopes. Notice that in choice sets $C_1$ and $C_2$, the only context effect is being exerted by feature 1, while in $C_3$ and $C_4$, the only context effect is being exerted by feature 2. Despite the fact that the captured effects come from different sources, the intercepts of the linear models capturing these two effects are the same, namely $\theta_3$. 
\end{example}

Motivated by observing these varying intercepts in real data (see \Cref{fig:context_effect_example}), we propose a second model that decomposes the LCL into context effects exerted by each feature. For this reason, we call it the decomposed linear context logit (DLCL). In the language of choice set effect functions, we now have $d$ context effect functions $F_1, \dots, F_d$ where each $F_k$ only depends on the values of feature $k$. We also replace $\theta$ with $d$ base preference vectors $B_1, \dots, B_d$ (which we combine into a $d \times d$ matrix $B$; we use matrix subscripts to index columns) that grant us varying intercepts. This gives us $d$ different contextual utilities $B_1 + F_1(C), \dots, B_d + F_d(C)$ that we will combine through a mixture model. 

Making the same assumptions as for the LCL, we decompose each choice set effect function $F_k(C) = \frac{1}{\lvert C \rvert}\sum_{j\in C} f_k((x_j)_k)$ (notice that $f_k$ is a function of only the $k$th feature, $(x_j)_k$). Incorporating a linearity assumption (and storing context effects exerted by feature $k$ in the $k$th column of $A$), we find
\[F_k(C) = \frac{1}{\lvert C \rvert} \sum_{j \in C} f_k((x_j)_k) = \frac{1}{\lvert C \rvert} \sum_{j \in C} A_k (x_j)_k = A_k (x_C)_k.\]
 
We use mixture weights $\pi_1, \dots, \pi_d$ with $\sum_{k=1}^d \pi_k = 1$ to describe the relative strengths of the effects exerted by each feature. For instance, in \Cref{ex:dlcl}, if the context effect exerted by feature 1 influences preferences more than the context effect exerted by feature 2, we would have $\pi_1 > \pi_2$. The DLCL is then a mixture of $d$ logits, where each component captures the context effects exerted by a single feature. The DLCL has choice probabilities
\begin{equation}
	\Pr(i, C) = \sum_{k=1}^d \pi_k \frac{\exp\big([B_k + A_k (x_C)_k]^T x_i\big)}{\sum_{j \in C} \exp\left([B_k + A_k (x_C)_k]^T x_j\right)}.
\end{equation}
Notice that each component corresponds to a logit with contextual preferences $B_k + F_k(C)$. In \Cref{ex:dlcl}, the context effect exerted by feature 1 is captured by the linear model $B_1 + A_1(x_C)_1$, while the effect exerted by feature 2 has a different intercept as well as a different slope: $B_2 + A_2(x_C)_2$.

The DLCL model has $2d^2 + d$ parameters: $d^2$ each for $A$ and $B$, and $d$ for $\pi$. The matrix $A$ has the same interpretation as in the LCL, while $B_{pq}$ represents the importance of feature $p$ when feature $q$ is zero (i.e., the intercept of the linear context effect exerted on $p$ by $q$). As in the LCL, we could in theory replace the linear model in each component with any function of $(x_C)_k$, but linear functions are a tractable and interpretable starting point. Note that the DLCL is \emph{not} a mixture of LCLs in the way that mixed logit is a mixture of MNLs: each component in the DLCL only acccounts for the context effect of one feature, whereas the LCL accounts for all feature context effects at once.

Additionally, we can take two views of what the DLCL mixture represents: we can either think of each individual combining several components of their preferences in a mixture (in which case the whole mixture model applies to every chooser), or we can think that every individual belongs to single component of the mixture and the mixing parameters describe the prevalence of each type of individual in the population. These two persepctives result in the same model and likelihood function, but it is useful to take the second view when we develop an EM algorithm for parameter estimation (\cref{sec:em_algorithm}).

\section{Identifiability of the LCL}\label{sec:lcl_identifiability}
We provide three results characterizing the identifiability of the LCL. Most significantly, we prove a necessary and sufficient condition that exactly determines when the model is identifiable (\Cref{thm:lcl_identify}). However, the condition is somewhat hard to reason about, so we also prove a simple necessary (but not sufficient) condition (\Cref{prop:lcl_non_identifiable}) and a simple sufficient (but not necessary) condition (\Cref{prop:lcl_identify_simple}). These supporting results and their proofs give additional insight into the main theorem.   

Following \citeauthor{seshadri2019discovering}~\cite{seshadri2019discovering}, we use $\mathcal{C_D}$ to denote the set of unique choice sets appearing in the dataset $\mathcal D$ and we say that an LCL is \emph{identifiable} from a dataset if there do not exist two distinct sets of parameters $(\theta, A)$ and $(\theta', A')$ that produce identical probability distributions over every choice set $C \in \mathcal{C_D}$. 
In the following, $\otimes$ denotes the Kronecker product.
\begin{theorem}\label{thm:lcl_identify}
	A $d$-feature linear context logit is identifiable from a dataset $\mathcal D$ if and only if
	\begin{equation}\label{eq:span_condition}\Span \left\{
		\begin{bmatrix}
		x_C\\
		1
	\end{bmatrix} \otimes (x_i- x_C)
	\mid C \in \mathcal{C_D}, i \in C \right\} = \R^{d^2+d}.
	\end{equation}
\end{theorem}

The proof of this theorem relies on three lemmas. 

\begin{lemma}[\cite{seshadri2019discovering}, Appendix A]\label{lemma:beta_bijection}
	For any choice set $C$, there is a bijection between the choice probabilities $\{\Pr(i, C) \mid i \in C\}$ and the log probability ratios $\{\beta_{i,C} \mid i \in C\}$ defined by 
	\begin{equation}
		\beta_{i,C} = \log\left(\frac{\Pr(i, C)}{\left[\prod_{j \in C} \Pr(j, C)\right]^{\frac{1}{\lvert C \rvert}}}\right).
	\end{equation}
\end{lemma}

\begin{proof}
	We can compute $\beta_{i,C}$ given all choice probabilities in $C$ as defined above. To obtain probabilities given log probability ratios, take 
	\begin{align*}
		\frac{\exp(\beta_{i, C})}{\sum_{j \in C} \exp(\beta_{j, C})} &= \frac{\frac{\Pr(i, C)}{\left(\prod_{h \in C} \Pr(h, C)\right)^{\frac{1}{\lvert C \rvert}}}}{\sum_{j \in C}\frac{\Pr(j, C)}{\left(\prod_{h \in C} \Pr(h, C)\right)^{\frac{1}{\lvert C \rvert}}}}\\
		&= \frac{\Pr(i, C)}{\sum_{j \in C}\Pr(j, C)}\\
		&= \Pr(i, C).\qedhere
	\end{align*}
\end{proof}

This means we can prove identifiability from the $\beta$s rather than from choice probabilities. We can also get a simple form for $\beta_{i, C}$ under the LCL.

\begin{lemma}\label{lemma:lcl_beta}
	In the LCL, $\beta_{i, C} = (\theta + Ax_C)^T (x_i - x_C)$.
\end{lemma}
\begin{proof}
Define $\theta_C = \theta + Ax_C$ for brevity.
	\begin{align*}
	\beta_{i, C} &= \log\left(\frac{\Pr(i, C)}{\left(\prod_{h \in C} \Pr(h, C)\right)^{\frac{1}{\lvert C \rvert}}}\right)\\
	&= \log\left(\frac{\frac{\exp\left(\theta_C^T x_i\right)}{\sum_{j \in C} \exp\left(\theta_C^T x_j\right)}}{\left(\prod_{h \in C} \frac{\exp\left(\theta_C^T x_h\right)}{\sum_{j \in C} \exp\left(\theta_C^T x_j\right)}\right)^{\frac{1}{\lvert C \rvert}}}\right)\\
	&= \log\left(\frac{\exp\left(\theta_C^T x_i\right)}{\left[\prod_{h \in C}\exp\left(\theta_C^T x_h\right)\right]^{\frac{1}{\lvert C \rvert}}}\right)\\
	&= \theta_C^T x_i - \frac{1}{\lvert C \rvert}\sum_{h \in C} \theta_C^T x_j\\
	&= \theta_C^T (x_i -  x_C). \qedhere
\end{align*}
\end{proof}

Let $\Vector(A)$ denote the vectorization of the matrix $A$ (the vector formed by stacking the columns of $A$).
\begin{lemma}[Special case of the vec trick, \cite{roth1934direct}]\label{lemma:vector_product}
	For any vectors $x \in \R^m, y \in \R^n$ and matrix $A \in \R^{m\times n}$, $x^T A y = (y \otimes x)^T \Vector(A)$.
\end{lemma}

\begin{proof}
	\begin{align*}
		x^T A y 
		= \sum_{i=1}^m \sum_{j=1}^n A_{ij}x_i y_j
		&= \sum_{j=1}^n y_j \sum_{i=1}^m  A_{ij}x_i\\		
		&= \begin{bmatrix}
			y_1 x\\
			y_2x\\
			\vdots\\
			y_n x
		\end{bmatrix}^T \Vector(A)\\
		&= (y \otimes x)^T \Vector(A). \qedhere
	\end{align*}
\end{proof}

With these facts in hand, we are ready to prove \Cref{thm:lcl_identify}.
\begin{proof}[Proof of \Cref{thm:lcl_identify}]
	Consider the log probability ratio of an item $i$ appearing in choice set $C$:
	\begin{align*}
		\beta_{i, C} &= (x_i - x_C)^T(\theta + Ax_C) \tag{by \Cref{lemma:lcl_beta}}\\
		&= (x_i - x_C)^T \begin{bmatrix}A & \theta\end{bmatrix} \begin{bmatrix}
			 x_C\\
			1
		\end{bmatrix}\\
		&= \left( \begin{bmatrix}
		 x_C\\
		1
	\end{bmatrix} \otimes (x_i- x_C)\right )^T
		\Vector \left(\begin{bmatrix}A & \theta\end{bmatrix}\right). \tag{by \Cref{lemma:vector_product}}
	\end{align*}
	Let $m = |\{(i, C) \mid C \in \mathcal{C_D}, i \in C\}|$ be the number of distinct (item, choice set) pairs in the dataset. 
	Index these pairs from $1$ to $m$. We construct the following $m\times (d^2+d)$ linear system by stacking all the $\beta_{i, C}$ equations:
	\begin{align*}
		\begin{bmatrix}
			\left( \begin{bmatrix}
		 x_{C_1}\\
		1
	\end{bmatrix} \otimes (x_{i_1}- x_{C_1})\right )^T\\
			\vdots\\
			\left( \begin{bmatrix}
		 x_{C_m}\\
		1
	\end{bmatrix} \otimes (x_{i_m}- x_{C_m})\right )^T\\ 
		\end{bmatrix} \Vector \left(\begin{bmatrix}A & \theta \end{bmatrix}\right)
		= \begin{bmatrix}
			\beta_{i_1, C_1}\\
			\vdots\\
			\beta_{i_m, C_{m}}
		\end{bmatrix}.
	\end{align*}
	Supposing the choice probabilities are generated according to the LCL, this system is consistent (although it is highly overdetermined with a large dataset). Any solution to this system is a setting of the parameters $\theta, A$ that results in the observed log probability ratios (and therefore choice probabilities, by \Cref{lemma:beta_bijection}). Since we know the system is consistent, it has a unique solution (i.e., the LCL is identifiable) if and only if the rows of the matrix span $\R^{d^2+d}$.
\end{proof}
To better understand the span condition of \Cref{thm:lcl_identify}, we now provide a simple necessary condition for indentifiability. 
Recall that a set of vectors $\{x_0, \dots, x_d\} \subset \R^d$ is \emph{affinely independent} if the set of vectors $\{x_1 - x_0,\dots, x_d - x_0\}$ is linearly independent. 
This is equivalent to requiring that $$\Span\left\{\begin{bmatrix}
		x_0\\
		1
	\end{bmatrix} , \dots,
	\begin{bmatrix}
		x_d\\
		1
	\end{bmatrix} 
	\right\} = \R^{d+1}.$$

\begin{proposition}\label{prop:lcl_non_identifiable}
	No $d$-feature linear context logit is identifiable from a dataset $\mathcal D$ if it does not include a set of $d+1$ choice sets with affinely independent mean feature vectors.
\end{proposition}
\begin{proof}
Suppose that $ x_{C_1}, \dots,  x_{C_k}$ ($k < d+1$) is a maximal set of affinely independent mean feature vectors appearing in the dataset $\mathcal D$. In each one of these choice sets $C_i$, the choice probabilities are determined by $\theta_{C_i} = \theta + A x_{C_i}$. However, since $k < d+1$, there are infinitely many affine transformations $\theta+Ax_{C_i}$ that map every $x_{C_i}$ to its corresponding $\theta_{C_i}$. For any other choice set $C' \notin \{C_1, \dots, C_k\}$, we can express its mean feature vector as an affine combination $ x_{C'} = \sum_{i=1}^k \alpha_i  x_{C_i}$, 
where $\sum_{i=1}^k \alpha_i=1$. 
We then have $\theta_{C'} = \theta+A(\sum_{i=1}^k \alpha_i  x_{C_i}) = \sum_{i=1}^k \alpha_i (\theta+A x_{C_i}) = \sum_{i=1}^k \alpha_i \theta_{C_i}$, so any of the infinitely many affine transformations that correctly map $x_{C_i}$ to $\theta_{C_i}$ will also map $x_{C'}$ to $\theta_{C'}$. This means there are infinitely many parameter settings $\theta$ and $A$ that would result in the same choice probabilities, so the LCL is not identifiable.
\end{proof}
The span condition in \Cref{thm:lcl_identify} is made more difficult to reason about because of the coupling between individual feature vectors $x_i$ and mean feature vectors $x_C$. We therefore provide a simple sufficient condition for identifiability that decouples these requirements and is optimal in the number of distinct choice sets (by \Cref{prop:lcl_non_identifiable}).

\begin{proposition}\label{prop:lcl_identify_simple}
	If a dataset contains $d+1$ distinct choice sets $C_0, \dots, C_d$ such that
	\begin{enumerate}
		\item[i.] the set of mean feature vectors $\{ x_{C_0}, \dots, x_{C_d}\}$ is affinely independent (the necessary condition from \Cref{prop:lcl_non_identifiable}) and
		\item[ii.] in each choice set $C_i$, there is some set of $d+1$ items with affinely independent features,
	\end{enumerate}
	then we can uniquely identify a $d$-feature LCL. 
\end{proposition}
\begin{proof}
We will use differences in log probability ratios to first identify the choice set dependent utilities $\theta_C = \theta+A x_C$ in each choice set and then combine those to determine $\theta$ and $A$. 

To remove a dependence on mean feature vectors, consider the difference of two log probability ratios in the same choice set:
\begin{align*}
 	\beta_{i_1, C} - \beta_{i_2, C} &= \theta_C^T (x_{i_1} - x_C) - \theta_C^T (x_{i_2} - x_C) \tag{by \Cref{lemma:lcl_beta}}\\
	 &= \theta_C^T (x_{i_1} - x_{i_2}).
\end{align*} 
In order to identify the vector $\theta_C$, form the following linear system from $d$ such differences, all in the same choice set $C$:
\begin{align*}
	\begin{bmatrix}
		(x_{i_1} - x_{i_0})^T\\
		(x_{i_2} - x_{i_0})^T\\
		\vdots\\
		(x_{i_d} - x_{i_0})^T
	\end{bmatrix} \theta_C = \begin{bmatrix}
		\beta_{i_1, C} - \beta_{i_0, C}\\
		\beta_{i_2, C} - \beta_{i_0, C}\\
		\vdots\\
		\beta_{i_d, C} - \beta_{i_0, C}
	\end{bmatrix}
\end{align*}
If the rows of the matrix are linearly independent, then we can uniquely solve this system to find $\theta_C$. For this to be the case, we need the $d+1$ feature vectors $x_{i_0}, \dots, x_{i_d}$ to be affinely independent.

In order to recover $\theta$ and $A$, we need to solve the affine system $\theta+Ax_C = \theta_C$ for $\theta$ and $A$ given observations of $x_C$ and $\theta_C$. 
Affine transformations in $d$ dimensions are uniquely specified by their action on a set of $d+1$ affinely independent vectors. So, if we have $d+1$ observed choice sets $C_0, \dots, C_d$ whose mean feature vectors $x_{C_0}, \dots, x_{C_d}$ are affinely independent (and if we know $\theta_{C_0}, \dots, \theta_{C_d}$), then we can uniquely identify $\theta$ and $A$. As we have seen, we can find $\theta_{C_0}, \dots, \theta_{C_d}$ if each of $C_0, \dots, C_d$ has $d+1$ items with affinely independent feature vectors.
\end{proof}

This concludes our analysis of LCL identifiability. Intuitively, our results show that identifiability requires enough choice sets with sufficiently different mean features containing enough sufficiently different items (with some coupling between the two requirements). The necessary and sufficient condition of \Cref{thm:lcl_identify} is easily satisfied in practice if there are no redundant features, so a unique best-fit LCL can be recovered from most of the datasets we examine. We leave characterization of DLCL identifiability for future work, as even mixed logits 
have notoriously complex identifiability conditions~\cite{grun2008identifiability,zhao2016learning,chierichetti2018learning}.

\section{Estimation}\label{sec:estimation}
One of the strengths of discrete choice models is that they can be learned from choice data. Given a dataset $\mathcal D$ consisting of observations $(i, C)$, where $i$ was selected from the choice set $C$, we wish to recover the parameters of a model that best describe the dataset. In this section, we describe estimation procedures for the LCL and DLCL. First, we show that the likelihood function of the LCL is log-concave, so maximum likelihood estimation can be used to optimally estimate the model. On the other hand, the DLCL does not have a log-concave likelihood. To address this issue, we derive the expectation-maximization algorithm for the DLCL, which only requires optimizing convex subproblems.

\subsection{Maximum Likelihood Estimation}
We use maximum likelihood estimation as the primary method of estimating the LCL and DLCL. For any discrete choice model, the likelihood of the parameters $\theta$ given a dataset $\mathcal D$ is
\begin{equation}
	\mathcal L (\theta; \mathcal D) = \prod_{(i, C) \in \mathcal D} \Pr(i, C),
\end{equation}
where $\Pr(i, C)$ depends on $\theta$.
In MLE, we find the parameters $\theta$ that maximize $\mathcal L(\theta; \mathcal D)$. Equivalently, we can minimize the negative log-likelihood (NLL):
\begin{equation}
	-\ell (\theta; \mathcal D) = -\sum_{(i, C) \in \mathcal D} \log \Pr(i, C).
\end{equation}
The NLL of the linear context logit is:
\begin{align}
	-\ell (\theta, A; \mathcal D) &= -\smashoperator{\sum_{(i, C) \in \mathcal D}} \log \frac{\exp(\left[\theta + Ax_C\right]^T x_i)}{\sum_{j \in C}\exp(\left[\theta + Ax_C\right]^T x_j)}\\
	&= \smashoperator{\sum_{(i, C) \in \mathcal D}}  -\left(\theta + Ax_C\right)^T x_i + \log \sum_{j \in C}\exp(\left[\theta + Ax_C\right]^T x_j).\label{eq:lcl_nll}
\end{align}
The LCL's negative log-likelihood function is convex (equivalently, the likelihood is log-concave). To see this, notice that the first term in the summand of \eqref{eq:lcl_nll} is a linear combination of entries of $\theta$ and $A$, so it is jointly convex in $\theta$ and $A$. Meanwhile, log-sum-exp is convex and monotonically increasing, so its composition with the linear functions $\left[\theta + Ax_C\right]^T x_j$ is also convex. We then have that $-\ell (\theta, A; \mathcal D)$ is convex, as the sum of convex functions is convex. Moreover, the second partial derivatives of the NLL function are all bounded (by a constant depending on the dataset), so its gradient is Lipschitz continuous. We can therefore use gradient descent to efficiently find a global optimum of $-\ell (\theta, A; \mathcal D)$.

On the other hand, the NLL of the DLCL (like that of the mixed logit) is not convex, so we can only hope to find a local optimum with gradient descent. 
Nonetheless, we find that it performs well in practice. One detail slightly complicates the MLE procedure for DLCL, namely the constraints $\sum_{k=1}^d \pi_k = 1$ and $\pi_k \ge 0$. We avoid this issue by replacing each mixture proportion $\pi_i$ with the softmax $\exp(\pi'_i) / \sum_{k=1}^d \exp(\pi'_k)$, with no constraints on $\pi_i'$. This function is differentiable (making it amenable to automatic differentiation software) and naturally enforces both the sum and non-negativity constraints.

\subsection{EM Algorithm for the DLCL}\label{sec:em_algorithm}
A standard method for estimating mixture models like the mixed logit is expectation-maximization (EM)~\cite{dempster1977maximum,train2009discrete}. 
Here, we derive the EM algorithm for the DLCL (see \cite{hastie2009elements} for a general treatment of EM algorithms). 
In this section, we use $\mathcal D_h$ to refer to the $h$th observation $(i, C)$. We use $\Delta_h \in \{1, \dots, d\}$ to denote the latent mixture component that the observation $\mathcal D_h$ comes from (taking the view mentioned earlier about each observation belonging to one component).

The EM algorithm is an iterative procedure that begins with initial guesses for the parameters $\theta^{(0)} = (A^{(0)}, B^{(0)}, \pi^{(0)})$ and updates them until convergence. In the update step, we maximize the expectation of the log-likelihood $\ell(A, B; \mathcal D, \Delta)$ over the distribution of the unobserved variable $\Delta$ conditioned on the observations $\mathcal D$ and the current estimates of the parameters, denoted $\E_\Delta[\ell(A, B; \mathcal D, \Delta)\mid \mathcal D, \theta^{(t)}]$. The new estimates $A^{(t+1)}$ and $B^{(t+1)}$ are the maximizers of this function. The new estimate of the mixture proportions $\pi^{(t+1)}$ has a closed form based on the probability that each observation comes from each mixture component according to the current estimates of $A$ and $B$. See \Cref{alg:em} for the complete procedure. We derive the details here, starting with a breakdown of the expectation function:
\begin{equation}
	\begin{aligned}
		\E_\Delta&[\ell(A, B; \mathcal D, \Delta)\mid \mathcal D, \theta^{(t)}]\\
		&= \sum_{(i, C) = \mathcal D_h} \sum_{k=1}^d \Pr(\Delta_h = k\mid i, C, \theta^{(t)}) \log  \Pr(i, C\mid \Delta_h=k, A, B).
	\end{aligned}
 	 \label{eq:emfunction}
 \end{equation} 
We can compute the first term in the summand (the \emph{responsibilities} describing how well each mixture component describes each observation) using Bayes' Theorem:
\begin{align}
	\Pr(\Delta_h = k\mid i, C, \theta^{(t)}) &= \frac{\Pr(i, C \mid \Delta_h = k, \theta^{(t)})\Pr(\Delta_h = k \mid \theta^{(t)})}{\Pr(i, C\mid \theta^{(t)})}\\
	&= \pi_{h}^{(t)}\frac{\Pr(i, C \mid \Delta_h = k, \theta^{(t)})}{\Pr(i, C \mid \theta^{(t)})}. \label{eq:responsibilities}
\end{align}
The numerator of \Cref{eq:responsibilities} is just the $k$th component of the DLCL choice probability (with our estimates for $A$ and $B$):
\begin{align}
	\Pr(i, C \mid \Delta_h = k, \theta^{(t)}) &= \frac{\exp\big([B_k^{(t)} + A_k^{(t)} (x_C)_k]^T x_i\big)}{\sum_{j \in C} \exp\big([B_k^{(t)} + A_k^{(t)} (x_C)_k]^T x_j\big)}.
\end{align}
Meanwhile, the denominator of \Cref{eq:responsibilities} is the sum of these probabilities weighted by the mixture weight estimates:
\begin{align}
	\Pr(i, C \mid \theta^{(t)}) &= \sum_{k=1}^d \pi_k^{(t)} \Pr(i, C \mid \Delta_h = k, \theta^{(t)}).
\end{align}
The last term in \Cref{eq:emfunction} is a function of the parameters $A, B$ (not their estimates):
\begin{align}
	&\log  \Pr(i, C\mid \Delta_h=k, A, B) = \log \left[\frac{\exp\big([B_k + A_k (x_C)_k]^T x_i\big)}{\sum_{j \in C} \exp\left([B_k + A_k (x_C)_k]^T x_j\right)}\right]\\
	&= [B_k + A_k (x_C)_k]^T x_i - \log \sum_{j \in C} \exp\left([B_k + A_k (x_C)_k]^T x_j\right). \label{eq:emlogprob}
\end{align}
\Cref{eq:emlogprob} is concave by the same reasoning that the LCL's NLL (\Cref{eq:lcl_nll}) is convex. Thus, the expectation $\E_\Delta[\ell(A, B; \mathcal D, \Delta)\mid \mathcal D, \theta^{(t)}]$, being the sum of positively scaled concave functions, is also concave. Its gradient is also Lipschitz continuous, just like the LCL's NLL. We can therefore find a global maximum using gradient ascent (in practice, we use gradient \emph{descent} to \emph{minimize} $-Q(A, B \mid \theta^{(t)})$). 

A major advantage of the EM algorithm over MLE for the DLCL is that it only requires optimizing convex functions with Lipschitz continuous gradients. 
However, while the EM algorithm is guaranteed to improve the log-likelihood of the parameter estimates at each step \cite{hastie2009elements}, it may still arrive at a local maximum. In practice, we find that the EM algorithm for DLCL outperforms stochastic gradient descent in 18/22 of our datasets (\Cref{tbl:em_comparison} in \Cref{sec:empirical_appendix}). 

\begin{algorithm}[tb]
   \caption{EM algorithm for estimating DLCL parameters.}
   \label{alg:em}
\begin{algorithmic}[1]
\algsetup{
  indent=1.5em,
  linenosize=\scriptsize,
  linenodelimiter=
}

  \STATE {\bfseries Input:} $m$ observations $\mathcal D$, $d$ features
  \STATE $A^{(0)}, B^{(0)}  \gets $ $d\times d$ randomly initialized matrices
  \STATE $\pi^{(0)} \gets $ $d$-dimensional vector with all entries equal to $\frac{1}{d}$
  \STATE $t \gets 0$
  \WHILE{not converged}
    \STATE $p_{hk} \gets \frac{\exp\big([B_k^{(t)} + A_k^{(t)} (x_C)_k]^T x_i\big)}{\sum_{j \in C} \exp\big([B_k^{(t)} + A_k^{(t)} (x_C)_k]^T x_j\big)}$ \\ \quad for each $(i, C) = \mathcal D_h$ and $k = 1, \dots, d$
    \STATE $r_{hk} \gets \frac{\pi_k^{(t)}p_{hk}}{\sum_{g=1}^d \pi_g^{(t)} p_{hk}}$ for each $h =1, \dots, m$ and $k = 1, \dots, d$
    \STATE \begin{equation*}
    	\begin{aligned}	
    		\textstyle Q(A, B \mid \theta^{(t)}) \gets &\sum_{(i, C) = \mathcal D_h} \sum_{k = 1}^d r_{hk} \Big[[B_k + A_k (x_C)_k]^T x_i \\
    		&- \log \sum_{j \in C} \exp\left([B_k + A_k (x_C)_k]^T x_j\right) \Big]
    	\end{aligned}	
    \end{equation*}
    \STATE Find a minimizer $A^*, B^*$ of $-Q(A, B \mid \theta^{(t)})$ using gradient descent
    \STATE $A^{(t+1)} \gets A^*$
    \STATE $B^{(t+1)} \gets B^*$
    \STATE $\pi_k^{(t+1)} \gets \frac{1}{|\mathcal D|} \sum_{h=1}^{|\mathcal D|} r_{hk}$ for each $k = 1, \dots, d$
	\STATE $t \gets t+1$	
  \ENDWHILE
  \RETURN $A^{(t)}, B^{(t)}, \pi^{(t)}$
\end{algorithmic}
\end{algorithm}

\section{Empirical Analysis}\label{sec:empricial_analysis}
We apply our LCL and DLCL models to two collections of real-word choice datasets. First, we examine datasets specifically collected to understand choice in various domains, including car purchasing and hotel booking. The features describing items naturally differ in each of these datasets. The second collection of datasets comes from a particular choice process in social networks, namely the formation of new connections. Here, we use graph properties as features (such as in-degree, a proxy for popularity~\cite{moody2011popularity}), allowing us to compare social dynamics across email, SMS, trust, and comment networks. In both dataset collections, we first establish that context effects occur and that our models better describe the data than traditional context-effect-free models, MNL and mixed logit. We then show how the learned models can be interpreted to recover intuitive feature context effects. Our code and results (including scripts for generating all plots and tables appearing in this paper) are available at \url{https://github.com/tomlinsonk/feature-context-effects} and links to our datasets are in \cref{sec:dataset_details}.

\xhdr{Estimation details}
For prediction experiments, we use 60\% of samples for training, 20\% for validation, and 20\% for testing. When testing model fit using likelihood-ratio tests, we estimate models from the whole dataset. We use the PyTorch implementation of the Adam optimizer~\cite{kingma2015adam,paszke2019pytorch} for maximum likelihood estimation with batch size 128 and the \texttt{amsgrad} flag \cite{reddi2018convergence}. We run the optimizer for $500$ epochs or 1 hour, whichever comes first. For the whole-model fits, we use weight decay $0.001$ and perform a search over the learning rates $0.0005, 0.001, 0.005, 0.01, 0.05, 0.1$, selecting the value that results in the highest likelihood model. For our prediction experiments, we perform a grid search over the weight decays $0, 0.0001, 0.0005, 0.001, 0.005, 0.01$ and the same learning rates as above. Here, we select the pair of weight decay and learning rate that achieves the best likelihood on the validation set.%
\footnote{More specifically, each hyperparameter setting produces a sequence of likelihoods by evaluating the likelihood on the validation data at the end of each training epoch. To account for noise, we select the hyperparameter setting where the minimum likelihood over the last five steps is maximal.}
Predictions are evaluated on the held-out test set. 
We use $d$ (the number of features) components for mixed logit to provide a fair comparison against DLCL (which always uses $d$ components).

\subsection{General Choice Datasets}
\begin{table}[t]
\centering
\caption{General choice datasets summary.}
\label{tbl:general_datasets}
\begin{tabular}{lrrrr}
\toprule
\textbf{Dataset} &
\textbf{Choices}&
\textbf{Features} &
\textbf{Largest Choice Set}
\\
\midrule
\textsc{district} & 5376 & 27 & 2\\
\textsc{district-smart} & 5376 & 6 & 2\\
\textsc{sushi} & 5000 & 6 & 10\\
\textsc{expedia} & 276593 & 5 & 38 \\
\textsc{car-a} & 2675 & 4 & 2\\
\textsc{car-b} & 2206 & 5 & 2\\
\textsc{car-alt} & 4654 & 21 & 6\\
\bottomrule
\end{tabular}
\end{table}
We analyze six choice datasets coming from online and survey data (\Cref{tbl:general_datasets}). The classic \textsc{sushi} dataset \cite{kamishima2003nantonac} includes surveys in which each respondent ranked 10 sushi (randomly selected from a set of 100 options) from favorite to least favorite. We consider the top ranked sushi to be the choice from the set of 10 options. The \textsc{expedia} dataset \cite{kaggle2013expedia} comes from online hotel booking. It contains user searches, displayed results, and which hotel was booked. We consider the set of search results to be the choice set and the booked hotel to be the choice. The \textsc{district} dataset \cite{kaufman2017measure,bower2020salient} contains pairwise comparisons between US congressional district shapes, with geometric properties as features. Survey respondents were asked to select which district was more ``compact'' (towards an understanding of anti-gerrymandering laws). The \textsc{district-smart} dataset is identical, but contains the subset of features identified by the authors of \cite{kaufman2017measure} as ``good predictors of compactness.'' The \textsc{car-a} and \textsc{car-b} datasets \cite{abbasnejad2013learning} contain pairwise comparisons between hypothetical cars described by features such as body type (SUV, sedan) and transmission (manual, automatic). \textsc{car-alt} \cite{brownstone1996transactions,mcfadden2000mixed} is similar, but has choice sets of six hypothetical cars and focuses on alternative-fuel vehicles (e.g., electric, compressed natural gas).

In all datasets, we standardize the features to have zero mean and unit variance, which allows us to more meaningfully compare learned parameters across datasets. \Cref{sec:dataset_details} has more details about the dataset features and preprocessing steps. The LCL is identifiable in \textsc{district-smart}, \textsc{expedia}, and \textsc{sushi}, but not the other general choice datasets (\Cref{sec:dataset_lcl_identify}). However, the $L_2$ regularization we apply (in the form of Adam weight decay) identifies the model in all cases.

\subsection{Network Datasets}\label{sec:network_datasets}
The general choice datasets above come with their own specialized set of features. For this reason, it is not possible to compare feature context effects across them. However, finding common patterns across datasets is one key step in showing that these effects are worth studying for their insight into human behavior as well as for their theoretical interest or use in prediction. To this end, we also study a collection of temporal social network datasets, where the choices are which edge to form and item features are graph properties of the nodes. This setting allows us to examine comparable context effects across thirteen datasets. In addition to providing datasets with identical features, this social network study is also of interest for insight into sociological processes and highlights how our models can be applied to a particular domain.

Recent work~\cite{overgoor2019choosing} showed that many models of network growth can be viewed through the lens of discrete choice. In a directed graph, the formation of the edge $u \rightarrow v$ can be thought of as a choice by the node $u$ to initiate new contact with $v$ (the graph might be a citation, communication, or friendship network, for example). The set from which $u$ chooses can vary, including all nodes in the graph or perhaps only a subset of closely available nodes. We focus specifically on (directed) \emph{triadic closure} \cite{rapoport1953spread,granovetter1977strength,easley2010networks}, where the node $u$ closes a triangle $u \rightarrow v \rightarrow w$ by adding the edge $u \rightarrow w$. This phenomenon is used in many influential network growth models~\cite{jin2001structure,holme2002growing,vazquez2003growing} and real-world networks show evidence of triadic closure in the form of high clustering coefficients \cite{holland1971transitivity,watts1998collective} and closure coefficients~\cite{yin2019local,yin2020measuring}.

\xhdr{Identifying choices from temporal network data}
Our network analysis assumes that the graphs grow according to a multi-mode model that combines triadic closure with a method of global edge formation. In particular, we assume that at each step, an initiating node either decides to form an edge to any node in the graph with probability $r$ or decides to close a triangle with probability $1-r$. This is the same setup used by the Jackson--Rogers model \cite{jackson2007meeting,holme2002growing} and the more general $(r, p)$-model \cite{overgoor2019choosing}. We focus only on the instances where a node decided to close a triangle, and assume that the node $u$ first picks one of its neighbors $v$ uniformly at random before choosing one of $v$'s neighbors as a new connection to initiate.

With this setup, we can reconstruct choice sets for each triangle closure in a directed temporal network dataset. 
In our setup, each time we observe a new edge $u \rightarrow w$ that closes at least one previously unclosed triangle, 
we model this as first selecting an intermediate $v$ uniformly at random through which the triangle was closed 
(note that $u \rightarrow w$ can close multiple unclosed triangles $u \rightarrow v \rightarrow w$ and $u \rightarrow v' \rightarrow w$). We then consider the choice set for that triangle closure to be the out-neighbors of $v$ that are not out-neighbors of $u$. For example, in a friendship network, this triangle closure could occur by $u$ attending a party hosted by $v$ at which they choose which of $v$'s friends to become friends with themselves.
Since we do not know from the network data whose party $u$ attended, 
we model the selection among possible intermediaries $v$s uniformly at random, 
conditioned on observing the triangle closure $u \rightarrow w$.
(One could model the intermediary selection of $v$ differently, say by weighting recent connections more heavily,
but we do not do this here.)

\xhdr{Node features}
The features of each node in the choice set are computed at the instant before the edge is closed (the features we consider evolve as the network grows over time). In our network datasets, we have timestamps on each edge and an edge may be observed many times (e.g., in an email network, $u$ may send $w$ many emails). The number of times an edge is observed is its \emph{weight}; an edge not appearing in the graph has weight 0. We use six features to describe each node $w$ that could be selected by the chooser $u$:
\begin{enumerate}
	\item \emph{in-degree}: the number of edges entering the target node $w$,
	\item \emph{shared neighbors}: the number of in- or out-neighbors of $u$ that are also in- or out-neighbors of $w$,
	\item \emph{reciprocal weight}: the weight of the reverse edge $u \leftarrow w$,
	\item \emph{send recency}: the number of seconds since $w$ initiated any outgoing edge,
	\item \emph{receive recency}: the number of seconds since $w$ received any incoming edge, and
	\item \emph{reciprocal recency}: the number of seconds since the reverse edge $u \leftarrow w$ was last observed.
\end{enumerate}
Following \citeauthor{overgoor2020scaling}~\cite{overgoor2020scaling}, we log-transform features 1 and 2. We take $\log (1+\text{feature 3})$ to handle weight 0 (by construction, in-degree and shared neighbors are never 0, since $v$ is always a shared neighbor of $u$ and $w$). Lastly, we transform the three temporal features with $\log^{-1}(2+\text{feature})$ and set them to 0 if the event has never occurred. This transformation ensures that (1) we can handle 0 seconds since the last event, (2) higher values mean more recency, and (3) ``no occurrence'' results in the lowest possible value of the transformed feature.

\begin{table}[t]
\centering
\caption{Network datasets summary.}
\label{tbl:network_datasets}
\begin{tabular}{lrrr}
\toprule
\textbf{Dataset} &
\textbf{Nodes} &
\textbf{Edges}&
\textbf{Triangle closures}\\
\midrule
\textsc{synthetic-mnl} & 1000& 391294 & 50000\\
\textsc{synthetic-lcl} & 1000 & 380584 & 50000  \\
\textsc{email-enron} & 18592 & 53477 & 19900 \\
\textsc{email-eu} & 986 & 24929 & 19603 \\
\textsc{email-w3c} & 20082 & 33409 & 3271  \\
\textsc{sms-a} & 44430 & 68834 & 6311 \\
\textsc{sms-b} & 72146 & 100974 & 9376 \\
\textsc{sms-c} & 14433 & 23285 & 2732 \\
\textsc{bitcoin-alpha} & 3783 & 24186 & 8823 \\
\textsc{bitcoin-otc} & 5881 & 35592 & 12750 \\
\textsc{reddit-hyperlink} & 23499 & 91946 & 37115 \\
\textsc{wiki-talk} & 22067 & 81125 & 27505  \\
\textsc{facebook-wall} & 46952 & 274086 & 68776  \\
\textsc{mathoverflow} & 24818 & 239978 & 137455  \\
\textsc{college-msg} & 1899 & 20296 & 6267 \\
\bottomrule
\end{tabular}
\end{table}

\xhdr{Network datasets} 
We processed thirteen network datasets, including three email datasets (\textsc{email-enron} \cite{benson2018simplicial}, \textsc{email-eu} \cite{leskovec2007graph,yin2017local,benson2018simplicial}, \textsc{email-w3c} \cite{Craswell-2005-TREC,benson2018found}); three SMS datasets (\textsc{sms-a}, \textsc{sms-b}, and \textsc{sms-c} \cite{wu2010evidence}), two Bitcoin trust datasets (\textsc{bitcoin-alpha} and \textsc{bitcoin-otc} \cite{kumar2016edge,kumar2018rev2}), an online messaging dataset (\textsc{college-msg} \cite{panzarasa2009patterns}), a hyperlink dataset (\textsc{reddit-hyperlink} \cite{kumar2018community}), and three online forum datasets (\textsc{facebook-wall} \cite{viswanath2009evolution}, \textsc{mathoverflow} \cite{paranjape2017motifs}, and \textsc{wiki-talk} \cite{leskovec2010signed,leskovec2010predicting}). All of these datasets are publicly available online (\cref{sec:dataset_details} has links to our preprocessed data as well as descriptions of the preprocessing steps and original sources.)

In addition to these real-world networks, we generate two synthetic networks to validate our identification of context effects, \textsc{synthetic-mnl} and \textsc{synthetic-lcl}. To generate these networks, we begin with 1000 isolated nodes. At each step, we add an edge uniformly at random with probability 0.9. With probability 0.1, we instead close a triangle by selecting a node $u$ and one of its neighbors $v$ uniformly at random. We then use either an MNL (for \textsc{synthetic-mnl}) or LCL (for \textsc{synthetic-lcl}) to choose which triangle $u \rightarrow v \rightarrow \,?$ to close (if there are no triangles to close starting from $u$, we add a random edge). We use the same features as in the real-world datasets, with Poisson-distributed simulated timestamp gaps between successive edges. We repeat this process until we 50000 triangles are closed
(\cref{sec:dataset_details} has details of the arbitrary parameters we used for the MNL and LCL).

\Cref{tbl:network_datasets} summarizes the network data. Using \Cref{thm:lcl_identify}, we find that the LCL is uniquely identifiable in every network dataset (\Cref{sec:dataset_lcl_identify}). 
Whereas we split the general choice datasets into training, validation, and testing sets uniformly at random, 
we instead split the network datasets temporally so that future edges are predicted based on parameters estimated from past edges.

\begin{figure}[t]
\includegraphics[width=\columnwidth]{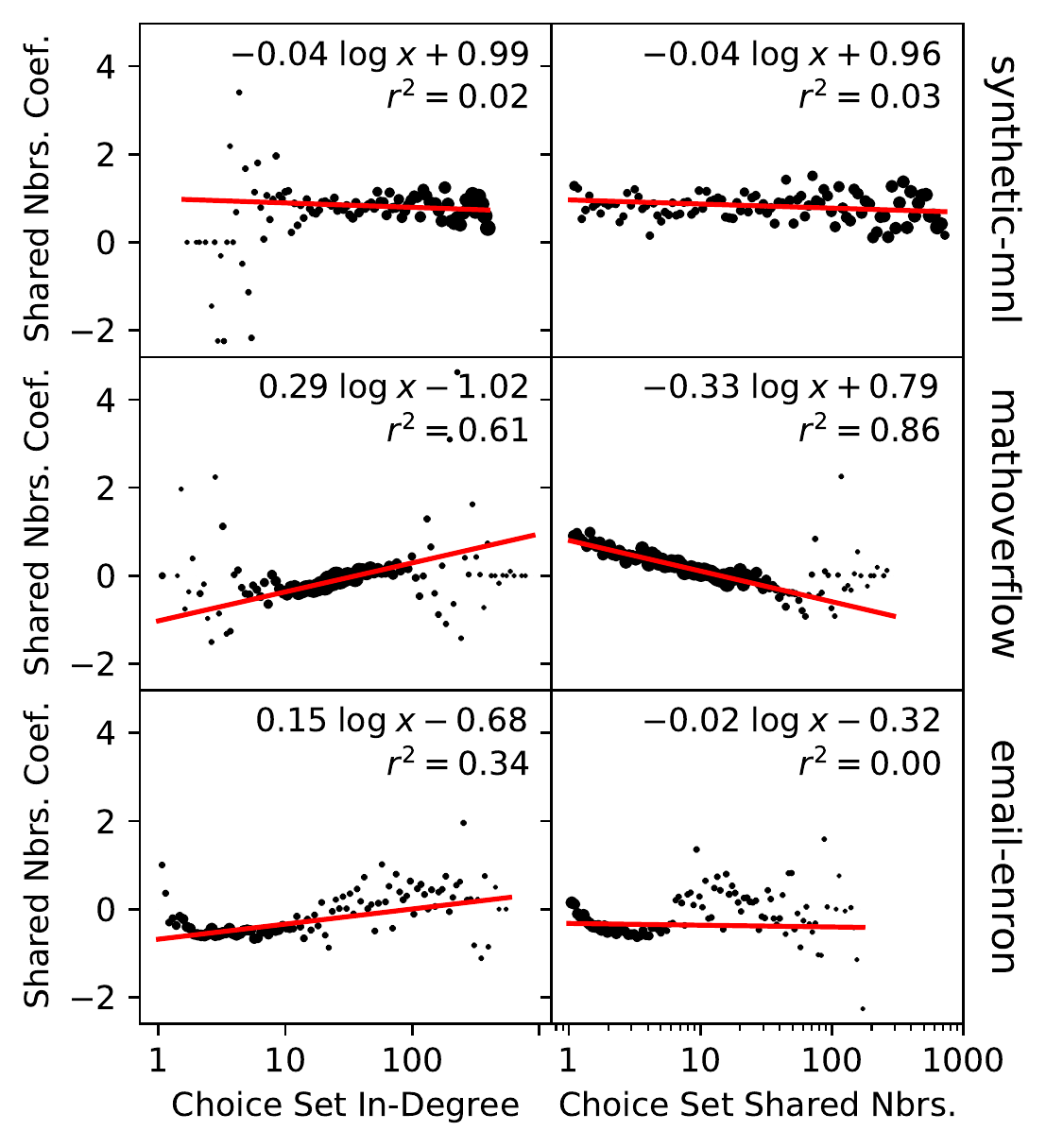}
\caption{Linear context effects observed in \textsc{mathoverflow} (middle row) and non-linear effects in \textsc{email-enron} (bottom row) in contrast with a context-effect-free synthetic dataset (top row). Each point shows the preference coefficient of the shared neighbors feature in choice sets with varying mean in-degree (left column) and shared neighbor counts (right column). These coefficients were learned by splitting observations into 100 bins according to their mean feature values and learning an MNL for each bin separately. The area of each point is proportional to the square root of the number of observations in its bin. The red lines are weighted least squares fits, with parameters and $r^2$ values shown in each subplot.}
\label{fig:context_effect_example}
\end{figure}

\begin{table}[t]
\centering
\caption{Dataset negative log-likelihoods. LCL and DLCL are significantly better fits than MNL and mixed logit in several general choice datasets and all real network datasets. This indicates the presence of feature context effects.}
\label{tbl:nll}
\begin{tabular}{lrrrr}
\toprule
& \textbf{MNL} & \textbf{LCL} & \begin{tabular}{r}
	\textbf{Mixed}\\\textbf{logit}
\end{tabular} & \textbf{DLCL}\\
\midrule
\textsc{district} & 3313 & \textbf{3130}\phantom{$^*$} & 3258 & 3206\phantom{$^{\dagger}$}\\
\textsc{district-smart} & 3426 & \textbf{3278}$^*$ & 3351 & 3303$^{\dagger}$\\
\textsc{expedia} & 839505 & 837649$^*$ & 839055 & \textbf{837569}$^{\dagger}$\\
\textsc{sushi} & 9821 & 9773$^*$ & 9793 & \textbf{9764}\phantom{$^{\dagger}$}\\
\textsc{car-a} & 1702 & 1694\phantom{$^*$} & 1696 & \textbf{1692}\phantom{$^{\dagger}$}\\
\textsc{car-b} & 1305 & 1295\phantom{$^*$} & 1297 & \textbf{1284}\phantom{$^{\dagger}$}\\
\textsc{car-alt} & 7393 & \textbf{6733}$^*$ & 7301 & 7011$^{\dagger}$\\
\midrule
\textsc{synthetic-mnl} & \textbf{210473} & 210486\phantom{$^*$} & 210503 & 210504\phantom{$^{\dagger}$}\\
\textsc{synthetic-lcl} & 140279 & \textbf{137232}$^*$ & 139539 & 137937$^{\dagger}$\\
\textsc{wiki-talk} & 99608 & 97748$^*$ & 95761 & \textbf{95134}$^{\dagger}$\\
\textsc{reddit-hyperlink} & 135108 & 132880$^*$ & 133766 & \textbf{132473}$^{\dagger}$\\
\textsc{bitcoin-alpha} & 19675 & 19190$^*$ & 19093 & \textbf{18877}$^{\dagger}$\\
\textsc{bitcoin-otc} & 26968 & 26101$^*$ & 25768 & \textbf{25348}$^{\dagger}$\\
\textsc{sms-a} & 8252 & \textbf{8056}$^*$ & 8239 & 8154$^{\dagger}$\\
\textsc{sms-b} & 13153 & \textbf{12823}$^*$ & 13147 & 12975$^{\dagger}$\\
\textsc{sms-c} & 4988 & 4880$^*$ & 4928 & \textbf{4871}$^{\dagger}$\\
\textsc{email-enron} & 73015 & 70061$^*$ & 71450 & \textbf{69254}$^{\dagger}$\\
\textsc{email-eu} & 53025 & 51822$^*$ & 51988 & \textbf{51431}$^{\dagger}$\\
\textsc{email-w3c} & 11012 & 10677$^*$ & 9898 & \textbf{9758}$^{\dagger}$\\
\textsc{facebook-wall} & 118208 & \textbf{116062}$^*$ & 117210 & 116328$^{\dagger}$\\
\textsc{college-msg} & 14575 & 14120$^*$ & 13849 & \textbf{13712}$^{\dagger}$\\
\textsc{mathoverflow} & 500537 & 479999$^*$ & 440482 & \textbf{435932}$^{\dagger}$\\
\midrule
\multicolumn{5}{l}{$^*$\footnotesize{Significant likelihood-ratio test vs.~MNL ($p<0.001$)}}\\
\multicolumn{5}{l}{$^{\dagger}$\footnotesize{Significant likelihood-ratio test vs.~mixed logit ($p< 0.001$)}}\\
\end{tabular}

\end{table}

\begin{table*}[h]
\centering
\caption{Test mean relative rank (lower is better), with standard deviations in parentheses. LCL and DLCL consistently outperform MNL and mixed logit, respectively. The differences are significant according to the Wilcoxon test in most of the datasets in which context effects are present. The best value for each dataset is bolded.}
\label{tbl:prediction}
\begin{tabular}{lccrrccrr}
\toprule{}
&\multirow{2}{*}{\textbf{MNL}} & \multirow{2}{*}{\textbf{LCL}} & \multicolumn{2}{c}{\textbf{Wilcoxon test}} & \multirow{2}{*}{\begin{tabular}{c}
\textbf{Mixed} \\\textbf{logit}
\end{tabular}} & \multirow{2}{*}{\textbf{DLCL}} & \multicolumn{2}{c}{\textbf{Wilcoxon test}}\\
\cmidrule{4-5}\cmidrule{8-9}
& & & \multicolumn{1}{c}{$T$} & \multicolumn{1}{c}{$p$-value} & & & \multicolumn{1}{c}{$T$} & \multicolumn{1}{c}{$p$-value}\\
\midrule
\textsc{district} & .3680 (.4823) & .3253\phantom{$^*$} (.4685) & 21692 & $0.0099$ & \textbf{.3188} (.4660) & .3225\phantom{$^{\dagger}$} (.4674) & 1783 & $0.67$\\
\textsc{district-smart} & .4006 (.4900) & .3764\phantom{$^*$} (.4845) & 13352 & $0.096$ & \textbf{.3271} (.4692) & .3448\phantom{$^{\dagger}$} (.4753) & 4736 & $0.12$\\
\textsc{expedia} & .3859 (.2954) & .3800$^*$ (.2945) & 54748886 & $< 10^{-16}$ & .3201 (.2825) & \textbf{.3195}$^{\dagger}$ (.2823) & 169205439 & $3.6 \times 10^{-12}$\\
\textsc{sushi} & .2727 (.2751) & .2737\phantom{$^*$} (.2781) & 8488 & $0.31$ & \textbf{.2724} (.2765) & .2732\phantom{$^{\dagger}$} (.2765) & 16841 & $0.83$\\
\textsc{car-a} & \textbf{.3570} (.4791) & \textbf{.3570}\phantom{$^*$} (.4791) & --- & $1.0$ & \textbf{.3570} (.4791) & \textbf{.3570}\phantom{$^{\dagger}$} (.4791) & --- & $1.0$\\
\textsc{car-b} & .3326 (.4711) & \textbf{.3213}\phantom{$^*$} (.4670) & 70 & $0.25$ & .3303 (.4703) & .3235\phantom{$^{\dagger}$} (.4678) & 63 & $0.47$\\
\textsc{car-alt} & .2944 (.2875) & \textbf{.2661}$^*$ (.2819) & 38042 & $0.00022$ & .2931 (.2966) & .2798\phantom{$^{\dagger}$} (.2837) & 15927 & $0.0068$\\
\midrule
\textsc{synthetic-mnl} & .1513 (.1865) & \textbf{.1512}$^*$ (.1863) & 6854099 & $5.2 \times 10^{-7}$ & .1513 (.1865) & \textbf{.1512}$^{\dagger}$ (.1862) & 8084911 & $4.4 \times 10^{-5}$\\
\textsc{synthetic-lcl} & .1360 (.1684) & \textbf{.1357}$^*$ (.1683) & 5928387 & $< 10^{-16}$ & .1360 (.1684) & .1358$^{\dagger}$ (.1683) & 5823072 & $< 10^{-16}$\\
\textsc{wiki-talk} & .2942 (.2914) & \textbf{.2620}$^*$ (.2736) & 3104783 & $< 10^{-16}$ & .3029 (.2989) & .2689$^{\dagger}$ (.2783) & 2826059 & $< 10^{-16}$\\
\textsc{reddit-hyperlink} & .2859 (.2611) & .2742$^*$ (.2576) & 5816298 & $< 10^{-16}$ & .2759 (.2566) & \textbf{.2721}$^{\dagger}$ (.2571) & 5223683 & $1.2 \times 10^{-6}$\\
\textsc{bitcoin-alpha} & .2724 (.3246) & .2504$^*$ (.3106) & 111647 & $< 10^{-16}$ & .2637 (.3248) & \textbf{.2502}$^{\dagger}$ (.3145) & 94940 & $< 10^{-16}$\\
\textsc{bitcoin-otc} & .1891 (.2756) & .1440$^*$ (.2385) & 166792 & $< 10^{-16}$ & .1732 (.2663) & \textbf{.1404}$^{\dagger}$ (.2322) & 170182 & $< 10^{-16}$\\
\textsc{sms-a} & .2825 (.3250) & \textbf{.2640}\phantom{$^*$} (.3204) & 44703 & $0.0028$ & .2823 (.3273) & .2777\phantom{$^{\dagger}$} (.3292) & 33843 & $0.56$\\
\textsc{sms-b} & .3045 (.3419) & \textbf{.2857}\phantom{$^*$} (.3331) & 139666 & $0.0099$ & .3042 (.3434) & .2904\phantom{$^{\dagger}$} (.3363) & 94000 & $0.0097$\\
\textsc{sms-c} & .3115 (.3455) & .3145\phantom{$^*$} (.3536) & 13741 & $0.28$ & \textbf{.3095} (.3479) & .3209\phantom{$^{\dagger}$} (.3586) & 8614 & $0.013$\\
\textsc{email-enron} & .1265 (.2068) & .1235$^*$ (.2117) & 1135919 & $3.3 \times 10^{-16}$ & .1191 (.1979) & \textbf{.1168}$^{\dagger}$ (.2014) & 1153907 & $1.2 \times 10^{-9}$\\
\textsc{email-eu} & .2683 (.3021) & .2635\phantom{$^*$} (.3016) & 1304640 & $0.051$ & .2756 (.3054) & \textbf{.2623}$^{\dagger}$ (.2986) & 1126935 & $4.3 \times 10^{-8}$\\
\textsc{email-w3c} & .1332 (.2070) & .1145$^*$ (.1771) & 35779 & $6.8 \times 10^{-9}$ & .1035 (.1714) & \textbf{.1006}\phantom{$^{\dagger}$} (.1662) & 33907 & $0.0014$\\
\textsc{facebook-wall} & .2176 (.2895) & \textbf{.2037}$^*$ (.2812) & 4916288 & $< 10^{-16}$ & .2170 (.2888) & .2102$^{\dagger}$ (.2868) & 4032776 & $1.2 \times 10^{-8}$\\
\textsc{college-msg} & .1850 (.2726) & \textbf{.1726}$^*$ (.2634) & 57859 & $1.6 \times 10^{-5}$ & .1839 (.2742) & .1727$^{\dagger}$ (.2606) & 63634 & $6.5 \times 10^{-6}$\\
\textsc{mathoverflow} & .1365 (.2486) & .1131$^*$ (.2171) & 42555886 & $< 10^{-16}$ & .1136 (.2152) & \textbf{.1097}$^{\dagger}$ (.2114) & 33272583 & $< 10^{-16}$\\
\midrule
\multicolumn{5}{l}{$^*$\footnotesize{Significant two-sided Wilcoxon signed-rank test vs.~MNL ($p<0.001$)}}\\
\multicolumn{5}{l}{$^{\dagger}$\footnotesize{Significant two-sided Wilcoxon signed-rank test vs.~mixed logit ($p<0.001$)}}\\
\end{tabular}
\end{table*}

\subsection{Results}\label{sec:results}
Our analysis of these datasets focuses on two questions: whether significant linear feature context effects appear in practice and if so, how we can identify and interpret them using our models. 

\xhdr{Binned MNLs for visualizing feature context effect} 
As a first step towards identifying whether linear context effects occur in datasets, we bin the samples of each dataset according to the mean values of each feature in the choice set. We then fit MNLs within each bin to examine whether the preference coefficients of features vary as the mean choice set features vary. \Cref{fig:context_effect_example} shows two clear linear (with the respect to the log-transformed feature) context effects in \textsc{mathoverflow}: (1) as the mean in-degree of the choice set increases, so does the shared neighbors preference coefficient and (2) the shared neighbors coefficient decreases in choice sets with higher mean shared neighbors. 
Colloquially, (1) close ties are a stronger predictor of new connections when selecting between a set of popular individuals and (2) common connections matter less when choosing from a closely connected group. The different intercepts of these two effects in \textsc{mathoverflow} also motivate decomposing the LCL to the DLCL. In addition, this figure shows evidence of non-linear context effects in \textsc{email-enron}, although we still model this with a linear effect.

\xhdr{Evaluating model fit}
Having seen in this example that context effects may be worth capturing, we compare the two models we introduce (LCL and DLCL) to the traditional choice models they subsume (MNL and mixed logit, respectively) with likelihood-ratio tests. By Wilks' theorem \cite{wilks1938large}, under the null hypothesis (i.e., that LCL/DLCL do not describe the data better than MNL/mixed logit) the distribution of twice the difference in the log-likelihoods of the models approaches a $\chi^2$-distribution with degrees of freedom equal to the difference in the number of parameters ($d^2$) as the number of samples grows. Note that this result relies on i.i.d. observations, which may not be the case in real datasets, so the test statistics should be interpreted with care. Additionally, our datasets are quite large, resulting in very small $p$-values --- this does not say anything about the effect size, only that we can be confident that some non-zero context effect is occurring. To correct for multiple hypotheses, we use $p<0.001$ as our significance threshold. 

\Cref{tbl:nll} shows the total NLL of every dataset under the four models, along with markers indicating the significant likelihood-ratio tests (the test statistics and $p$-values of each can be found in \Cref{sec:empirical_appendix}). The bolded entries indicate the best likelihood for each dataset. Validating our approach, the best likelihood for the \textsc{synthetic-mnl} dataset is achieved by MNL\footnote{Note that the true LCL optimum would have better likelihood than MNL, since LCL subsumes MNL, but the complexity introduced by additional parameters means that LCL does not beat MNL within the 500 training epochs.} and the best likelihood for \textsc{synthetic-lcl} is achieved by LCL. Moreover, the likelihood-ratio tests behave as we should expect on the two synthetic datasets: significant for \textsc{synthetic-lcl} ($p < 10^{-16}$), insignificant for \textsc{synthetic-mnl} ($p=1.0$). 
On the real network datasets, all likelihood-ratio tests are significant (all with $p<10^{-9}$), indicating that feature context effects are occurring. In the general choice datasets, \textsc{expedia} ($p < 10^{-16}$), \textsc{district-smart} ($p < 10^{-16}$), \textsc{sushi} ($p = 1.6\times 10^{-7}$), and \textsc{car-alt} ($p < 10^{-16}$) have significant context effects according to the LCL likelihood-ratio test. The \textsc{district} ($p=1.0$), \textsc{car-a} ($p=0.44$), and \textsc{car-b} ($p=1.0$) datasets do not have significant context effects according to the LCL likelihood-ratio test. 

The relative likelihoods of the models give us insight into what types of effects are more present in each dataset: LCL and DLCL both capture how the preference coefficients of features vary as a linear function of the mean choice set features, while mixed logit and DLCL both capture preferences that vary among subpopulations. Thus, if we see a large gap in likelihood between MNL/mixed logit and LCL/DLCL, but not between MNL and mixed logit or between LCL and DLCL, then we have substantial evidence for linear context effects, but not subpopulations with heterogeneous preferences (an alternative mechanism for IIA violations). As we would expect, this is exactly the pattern that appears for \textsc{synthetic-lcl}. This pattern also appears in \textsc{sms-a}, \textsc{sms-b}, \textsc{expedia}, and \textsc{car-alt}. More commonly, datasets show evidence of both linear context effects and heterogeneous subpopulations (for instance, both of these phenomena are pronounced in \textsc{mathoverflow}). 

\xhdr{Evaluating predictive power}
The likelihood-ratio tests provide strong evidence that feature context effects occur in many real-world choice datasets. A related question is whether capturing them improves out-of-sample prediction performance. To address this question, we measure the mean relative rank of the true selected item in the output ranking of each method.
More specifically, we define the \emph{relative rank} of an item $i$ to be its index when the choice set $C$ is sorted in descending probability order (with ties resolved by taking the mean of all possible indices), divided by $\lvert C \rvert-1$. If the true selected item has the highest probability, this results in a relative rank of $0$, and if it has the lowest probability, it has a relative rank of $1$. The mean relative rank over the test set is a measure of how good the model's predictions are, from 0 (best possible) to 1 (worst possible). We use this measure rather than \emph{mean reciprocal rank}, which is common in evaluating preference learning methods, because it suffers from the undesirable property of being a mean of ordinal values and can therefore produce non-intuitive results~\cite{fuhr2018some}. 
\Cref{tbl:prediction} reports the complete mean relative rank results.

In order to test whether the differences in mean relative ranks of the LCL (resp.~DLCL) and MNL (resp.~mixed logit) are significant, we apply a two-sided Wilcoxon signed-rank test~\cite{wilcoxon1945individual} (the differences in relative rank are not normally distributed, especially in pairwise comparison data, so we do not use a paired $t$-test), as implemented by \texttt{scipy}~\cite{virtanen2020scipy}. 
In the general choice datasets, LCL has significantly better predictive performance in \textsc{expedia} ($p < 10^{-16}$), and \textsc{car-alt} ($p=2.2\times 10^{-4}$), while DLCL is only significantly better than mixed logit in \textsc{expedia} ($p = 3.6 \times 10^{-12}$).
In all of the real networks except \textsc{email-eu} ($p=0.34$) and the three SMS datasets ($p=0.0028, 0.0099, 0.28$, respectively), LCL has significantly better predictive performance than MNL (all $p < 10^{-4}$).
On these datasets, DLCL has significantly better predictive performance than mixed logit in all network datasets (all $p < 0.0001$) except \textsc{sms-a} ($p=0.56$), \textsc{sms-b} ($p=0.0097$), \textsc{sms-c} ($p=0.013$) and \textsc{email-w3c} ($p=0.0014$). 
See \Cref{tbl:prediction} for detailed results, including test statistics.

These results illustrate that in almost all of the datasets in which context effects appear to be present (according to likelihood-ratio tests), accounting for them results in improved predictive power. In some cases, the improvement can be quite large, despite the fact that the models account for subtle effects: for example, in \textsc{bitcoin-otc}, the mean relative rank is $24\%$ lower in LCL than in MNL. In other cases, the improvement is more modest, as in \textsc{facebook-wall} ($6\%$ lower).

\xhdr{Interpreting learned models on general choice datasets}
The previous analyses of model fit and predictive power indicate that linear context effects are indeed a significant factor in our choice datasets. In this section, we investigate what these effects are and show how our models can be interpreted to discover choice behaviors indicated by the datasets. We focus on the LCL because of its simpler structure and convex objective (we can be more confident in qualitative analyses, knowing that our optimization procedure finds a (near) minimizer).

Recall that the learned context effect matrix $A$ contains entries $A_{pq}$ for each pair of features $p, q$. When $A_{pq}$ is positive, this means that higher values of $q$ in the choice set increase the preference coefficient of $p$ (and negative values indicate a decrease). In this notation, the \emph{column} is the feature \emph{exerting} the effect and the \emph{row} is the feature \emph{influenced} by the effect.

For the general choice datasets, we pick out three datasets for detailed examination: \textsc{expedia}, \textsc{car-alt}, and \textsc{sushi}. Both \textsc{expedia} and \textsc{car-alt} show large context effects that significantly improve predictive power. However, both of these datasets have choice sets that correlate with user preferences (discussed in more detail below), which likely contributes to these effects. On the other hand, choice sets in \textsc{sushi} are independent of respondent preferences~\cite{kamishima2003nantonac} and we also see significant differences in model likelihoods. 
The five context effects with largest magnitude in each dataset are shown in \Cref{tbl:expedia_biggest,tbl:caralt_biggest,tbl:sushi_biggest}. Note that features are all standardized, so picking the largest entries of $A$ is meaningful. 

While the magnitude of the effect in $A$ tells us the \emph{size} of the learned context effect, it does not indicate how \emph{significant} the effect is in determining choice. To measure this, we perform likelihood-ratio tests (against MNL) on a constrained LCL in which all but one of the entries in $A$ is set to zero. This allows us to determine how much an individual context effect contributes to the overall choice process. Note that the negative log-likelihood function of the LCL remains convex under this constraint, so we can still estimate the model well.\footnote{This demonstrates a useful application of the LCL that is similar to other sparse regressions or feature selection methods: we can enforce a candidate sparsity pattern on the context effect matrix, train the model, and determine the significance of that subset of effects.} We use the same hyperparameters and training procedure as before. The tables show the learned magnitude of the context effect in this contrained model, denoted $\overline{A}_{pq}$. Additionally, we display the likelihood-ratio test statsitic (LRT) and $p$-value in each table, indicating the statistical significance of the feature context effect.  

\begin{table}[h]
\caption{Five largest context effects in \textsc{sushi}.}
\label{tbl:sushi_biggest}
\begin{tabular}{lrrrr}
\toprule
\textbf{Effect ($q$ on $p$)} & \textbf{$A_{pq}$} & \textbf{$\overline{A}_{pq}$} & \textbf{LRT} & $p$-value\\
\midrule
\emph{popularity} on \emph{popularity} & $-0.28$ & $-0.11$ & $3.0$ & $0.081$\\
\emph{availability} on \emph{is maki} & $0.24$ & $0.04$ & $0.39$ & $0.53$\\
\emph{oiliness} on \emph{oiliness} & $-0.20$ & $-0.26$ & $23$ & $1.5 \times 10^{-6}$\\
\emph{popularity} on \emph{availability} & $0.19$ & $0.09$ & $2.3$ & $0.13$\\
\emph{availability} on \emph{oiliness} & $-0.18$ & $-0.04$ & $0.74$ & $0.39$\\
\bottomrule
\end{tabular}
\end{table}

First, we examine the \textsc{sushi} dataset, which benefits from randomly chosen choice sets. By far the most significant effect is that respondents given more oily sushi options showed a stronger aversion to oily sushi ($p = 1.5 \times 10^{-6}$). The randomization of choice sets allows us to hypothesize that this is a true causal effect: that too much oiliness on the menu makes oily foods less appealing, which could be an example of the similarity effect.
The other context effects with largest magnitude in $A$ are not significant on their own. 
Notice that the magnitudes of the other effects decrease significantly in the constrained model with little impact on the likelihood, a sure sign of null effects --- in constrast, the oiliness aversion effect \emph{increases} in magnitude in the constrained model. This example demonstrates why it is vital to not only consider learned parameters but
to also understand potential context effects of interest in isolation.
Post-selection inference for sparse models is another option for similar analyses~\cite{lee2016exact,taylor2018post}.

\begin{table}[h]
\caption{Five largest context effects in \textsc{expedia}.}
\label{tbl:expedia_biggest}
\begin{tabular}{lrrrr}
\toprule
\textbf{Effect ($q$ on $p$)} & $A_{pq}$ & \textbf{$\overline{A}_{pq}$} & \textbf{LRT} & $p$-value\\
\midrule
\emph{location score} on \emph{price} & $-0.47$ & $-0.13$ & $10$ & $0.002$\\
\emph{on promotion} on \emph{price} & $0.27$ & $0.13$ & $17$ & $3.5 \times 10^{-5}$\\
\emph{review score} on \emph{price} & $-0.19$ & $-0.13$ & $29$ & $8.6 \times 10^{-8}$\\
\emph{star rating} on \emph{price} & $0.15$ & $0.20$ & $65$ & $6.0 \times 10^{-16}$\\
\emph{price} on \emph{star rating} & $0.10$ & $0.03$ & $4.0$ & $0.046$\\
\bottomrule
\end{tabular}
\end{table}

In \textsc{expedia}, four of the of five the largest-magnitude effects are highly significant in isolation. The largest effect in the full model is a decrease in willingness to pay (i.e., cheaper options are more preferred) when the mean location score of the choice set is high ($p=0.002$). Additionally, if many of the options are marked as ``on promotion,'' people seem more willing to book higher priced hotels ($p=3.5 \times 10^{-5}$). Interestingly, when the available hotels tend to be well-reviewed by other Expedia users, people are more price-averse ($p=8.6 \times 10^{-8}$), but they are less price-averse when the available hotels tend to have high star ratings ($p=6.0 \times 10^{-16}$). This may be because people searching for five-star hotels are not looking for the cheapest options, whereas people searching for well-reviewed hotels are looking for good deals.\footnote{The dataset does not include this information, only the location, length of stay, booking window, adult/children count, and room count of the search. Expedia does allow for filtering by review, star rating, etc.~after a search is performed, but it is not clear whether the dataset includes searches that were subsequently filtered.} When interpreting these effects, it is important to keep in mind that the choice sets in \textsc{expedia} may be influenced by user preferences to begin with, so we cannot determine whether the effects are causal. Nonetheless, the learned LCL model lends as much insight into user behavior as we could hope from found choice data and could motivate a randomized controlled trial aimed at determining causal effects. It also illustrates an important point to keep in mind when using choice data from recommender systems: choice sets are not independent from preferences.

\begin{table}[h]
\caption{Five largest context effects in \textsc{car-alt}.}
\label{tbl:caralt_biggest}
\begin{tabular}{lrrrr}
\toprule
\textbf{Effect ($q$ on $p$)} & $A_{pq}$ & \textbf{$\overline{A}_{pq}$} & \textbf{LRT} & $p$-value\\
\midrule
\emph{truck} on \emph{truck} & $1.06$ & $0.83$ & $239$ & $< 10^{-16}$\\
\emph{van} on \emph{van} & $0.94$ & $0.97$ & $309$ & $< 10^{-16}$\\
\emph{suv} on \emph{station wagon} & $0.89$ & $0.98$ & $-0.21$ & $1.0$\\
\emph{station wagon} on \emph{station wagon} & $0.88$ & $0.93$ & $153$ & $< 10^{-16}$\\
\emph{sports car} on \emph{station wagon} & $0.86$ & $0.96$ & $-0.21$ & $1.0$\\
\bottomrule
\end{tabular}
\end{table}

In \textsc{car-alt}, the three of the largest context effects are positive self-effects of the \emph{truck}, \emph{van}, and \emph{station wagon} body types (the other possible types are \emph{car}, \emph{sports car}, and \emph{SUV}). These effects are highly significant according to the constrained LCL (all $p< 10^{-16}$). These are the least preferred (and ``non-sporty'') body types, with base preference coefficients of $-0.49$, $-0.13$, and $-0.23$, respectively (compared to $0.58$ for \emph{sports car} and $0.60$ for \emph{SUV}; \emph{car} is coded as the lack of any of these binary indicators, and therefore has an effective preference coefficient of $0.00$). The high significance of these effects is likely due to the fact that the surveys were designed to have vehicles with ``body types and prices [...] that were similar (but not identical) to the household’s description of their next intended vehicle purchase'' \cite{brownstone1996transactions}. For example, choice sets with more trucks were offered to people who stated they wanted a truck; we observe the result of this biased surveying, that in choice sets with more trucks, people are more likely to prefer trucks. Our model (and any other model that operates on observational choice data) can only identify correlations. This example illustrates another potential use of the LCL: discovering choice set assignment bias. The fact that choice sets were crafted to align with respondent prefrences in \textsc{car-alt} is not mentioned in an influential discrete choice paper using the dataset~\cite{mcfadden2000mixed}. We discovered the choice set assignment mechanism after the strong LCL self-effects caught our attention and motivated us to track down the original dataset description. Thus, the LCL brought an important (and previously overlooked) aspect of a dataset to light. 

\begin{figure}[!h]
\includegraphics[width=0.93\columnwidth]{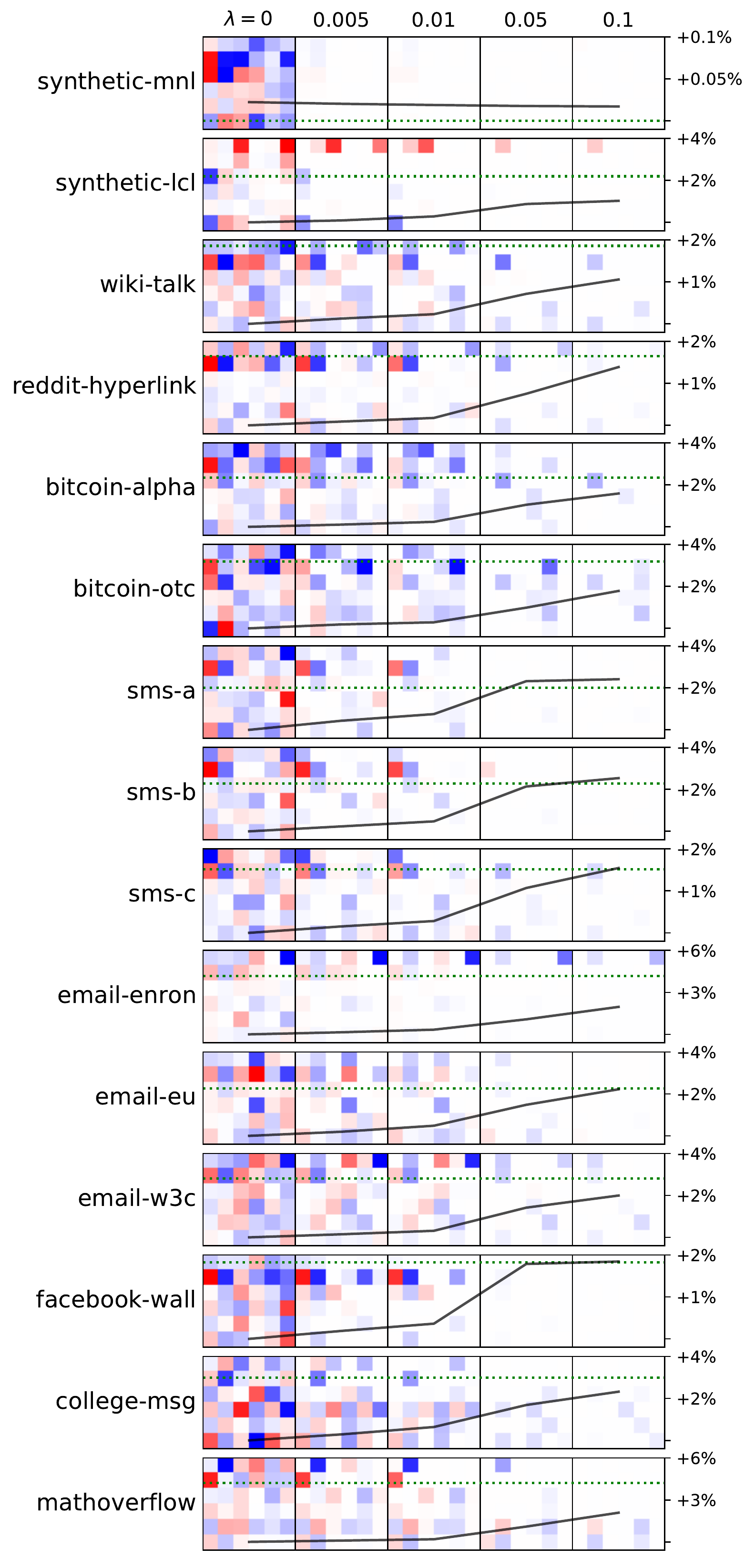}
\caption{Effect of $L_1$ regularization on the context effect matrix $A$ in the LCL. The parameter $\lambda$ (increasing left to right) controls the strength of the regularization. Each box visualizes the learned matrix $A$ (blue = negative, red = positive, white = zero; consistent color scales within but not between rows) at the regularization level specified by the column. The black line tracks the total NLL of the LCL (the \% on the y-axes is relative to the NLL of the best model plotted for that dataset). The dotted green line is the significance threshold of a likelihood-ratio test against an MNL ($p<0.001$; black line below the threshold means the LCL is a significantly better fit than MNL).}
\label{fig:l1}
\end{figure}

\begin{figure}[t]
\includegraphics[width=\columnwidth]{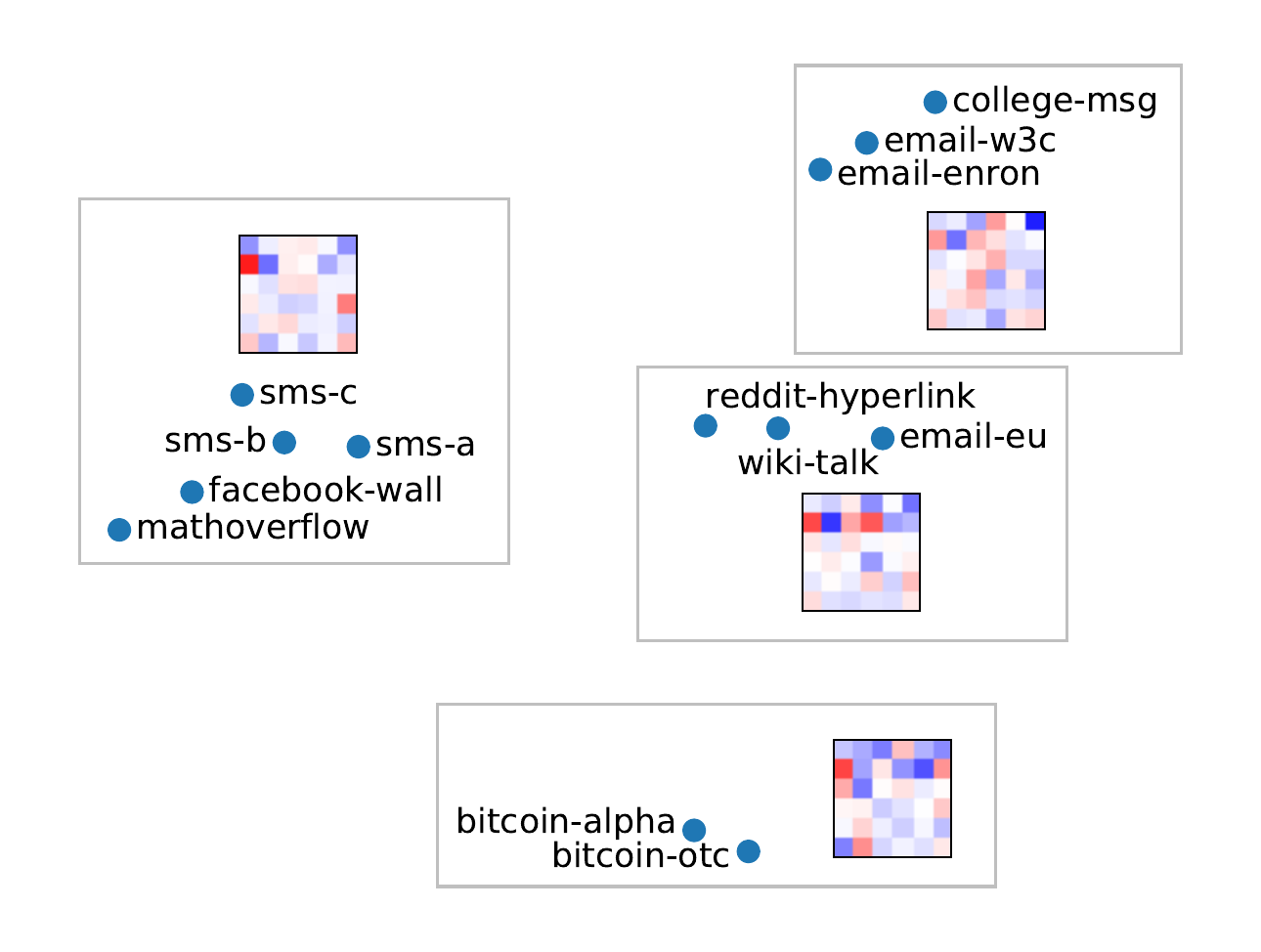}
\caption{t-SNE embedding of the learned LCL context effect matrices $A$ of the thirteen real network datasets. Each $6 \times 6$ matrix was vectorized and scaled to have unit norm before being embedded in two dimensions. The mean matrix $A$ in each of the four clusters is shown, with blue representing negative entries and red representing positive entries.}
\label{fig:tsne}
\end{figure}

\xhdr{Interpreting learned models on network growth datasets}
We take a different approach to examining significant context effects in the network datasets to showcase another useful application of the LCL. To visualize how much different effects influence choice in the network datasets, we apply $L_1$ regularization of varying strength to the LCL matrix $A$ during training (using the same hyperparameters for Adam), which encourages sparsity (\Cref{fig:l1}). We also visualize how this increasing sparsity impacts the likelihood of the model, and at what point the likelihood becomes no better than MNL according to a likelihood-ratio test.

\Cref{fig:l1} reveals several important effects shared by multiple datasets. For example, in \textsc{mathoverflow}, \textsc{facebook-wall}, \textsc{sms-a}, \textsc{sms-b}, \textsc{reddit-hyperlink}, feature 1 (in-degree) has a positive effect on the coefficient of feature 2 (shared neighbors). Colloquially, the data indicate that close connections matter more when choosing from a popular group. Another common effect appears in \textsc{email-enron} and \textsc{email-w3c}, a negative effect of feature 6 (reciprocal recency) on feature 1 (in-degree): high-volume email recipients are less likely to be targeted when the sender's inbox is full of recent messages from other potential targets. 

In both of these examples, when increasing regularization causes those entries of $A$ to go to 0, we see a jump in the likelihood, indicating that these are important effects to capture (note that we plot NLL, so lower is better). Additionally, we see in the top row how a dataset with no context effects (\textsc{synthetic-mnl}) behaves: $A$ immediately goes to 0 when any $L_1$ regularization is applied, without any worsening of the dataset likelihood.

Motivated by observing these shared effects across network datasets, we performed a t-SNE~\cite{maaten2008visualizing} embedding of the learned $A$ matrices (\Cref{fig:tsne}). As visualized in the embedding, datasets coming from similar domains tend to have similar context effects.

\section{Discussion}\label{sec:discussion}

Discovering intuitive context effects from choice data using our models has a number of possible applications. In recommender systems, insight into context effects could inform the set of options suggested to the user. Our findings could also serve as motivation for more controlled investigation of particular context effects in economics or psychology. A central contribution of our work is showing how intuitive and general context effects can be automatically recovered from observed choices and tested for significance. While we focused on linear context effects for simplicity, some of our datasets (e.g., \textsc{email-enron} in \Cref{fig:context_effect_example}) show evidence of non-linearity. A method of capturing these more complex effects (while retaining ease of training and interpretation) would be a valuable future contribution.

Our network data analysis demonstrates the presence of context effects in network growth, 
which can aid network growth modeling within the fields of network science and social network analysis.
We focused on triadic closure, which naturally provides small choice sets in which to observe the effect of context. Incorporating context effects in other modes of network growth (such as connections with unrelated nodes) is an interesting avenue for future research. One challenge this poses is that global modes of edge formation have much larger choice sets, requiring the use of negative sampling to preserve effective estimation \cite{overgoor2020scaling}. It is not immediately clear how to combine negative sampling with context effects over a large choice set.

One limitation of our approach (and any method operating on observational choice data) is that the generalizability of the effects we identify is constrained by correlations present in the data. For example, choice sets arising from recommender systems (such as \textsc{expedia}) are correlated with the preferences of their users by design. This makes it difficult to distinguish between how a user's preferences are affected by the choice set and how the user's preferences influence the choice set, particularly if we do not have access to the recommender system's process of choice set determination. In other cases, we might have random choice sets (as in \textsc{sushi}) or we might have no information about how choice sets are determined. In the latter case, our approach could also be used to find evidence of choice sets targeted at chooser preferences: if we observe many positive self-effects (e.g., preference for trucks is higher in choice sets with more trucks, as we saw in \textsc{car-alt}), this could mean that people's preferences are being catered to in their options. In some cases, this could be undesirable (e.g., if the party presenting individuals with options is supposed to be impartial), and our methods could provide a mechanism for identifying unwanted interventions.

Another challenging avenue for future work would be a method of discovering more complex relational context effects from choice data. The feature context effects we study describe the influence of one feature on another, but some of the traditional context effects studied in economics and psychology (e.g., the compromise effect) are based on the relationship between the features of several items. These effects are typically studied with targeted models that are hand-crafted to capture the desired effect. A general method of encoding and learning relational context effects could enable the discovery of new complex effects not yet envisioned by choice theorists, but nonetheless present in choice datasets. 

\section*{Acknowledgments}
This research was supported by ARO MURI, ARO Award W911NF19-1-0057, NSF Award DMS-1830274, and JP Morgan Chase \& Co. We thank Jan Overgoor and Johan Ugander for helpful conversations and Sophia Franco for naming suggestions.

\bibliographystyle{ACM-Reference-Format}
\bibliography{references}

\clearpage
\appendix

\section{Dataset Appendix}\label{sec:dataset_details}
Our preprocessed and documented versions of every dataset are available for download at \url{https://drive.google.com/file/d/1QAr-tCZ4OWRcrsQ0tHYwmTate5ED21PI/view}. Below, we describe the original sources for each dataset, along with feature descriptions for the general choice datasets and (where applicable) our preprocessing steps.

\subsection{General Choice Datasets}
\xhdr{\textsc{district}, \textsc{district-smart}} The \textsc{district} dataset was introduced in \cite{kaufman2017measure} and is scheduled to be uploaded to the Harvard Dataverse (\url{https://dataverse.harvard.edu}). At the time of writing, it is not yet available through the Dataverse. We obtained the dataset from Amanda Bower with permission from Aaron Kaufman. We used the preprocessed version that was used in \cite{bower2020salient}. The \textsc{district} dataset has 27 features that are all geometric properties of district shapes: \emph{points}, \emph{var\textunderscore{}xcoord}, \emph{var\textunderscore{}ycoord}, \emph{varcoord\textunderscore{}ratio}, \emph{avgline}, \emph{varline}, \emph{boyce}, \emph{lenwid}, \emph{jagged}, \emph{parts}, \emph{hull}, \emph{bbox}, \emph{reock}, \emph{polsby}, \emph{schwartzberg}, \emph{circle\textunderscore{}area}, \emph{circle\textunderscore{}perim}, \emph{hull\textunderscore{}area}, \emph{hull\textunderscore{}perim}, \emph{orig\textunderscore{}area}, \emph{district\textunderscore{}perim}, \emph{corners}, \emph{xvar}, \emph{yvar}, \emph{cornervar\textunderscore{}ratio}, \emph{sym\textunderscore{}x}, and \emph{sym\textunderscore{}y}. Detailed descriptions of these features can be found in~\cite[Appendix A]{kaufman2017measure}. We only use pairwise comparison data, not the ranking data that was also collected. The \textsc{district-smart} dataset uses only the six features \emph{hull}, \emph{bbox}, \emph{reock}, \emph{polsby}, \emph{sym\textunderscore{}x}, and \emph{sym\textunderscore{}y}, identified by \citeauthor{kaufman2017measure} to be ``good predictors of compactness''~\cite[Appendix F]{kaufman2017measure}.

\xhdr{\textsc{expedia}} Raw data downloaded from \url{https://www.kaggle.com/c/expedia-personalized-sort/overview}. We use the file \texttt{train.csv}. We select only the \texttt{srch\textunderscore{}id}s that result in a booking (\texttt{booking\textunderscore{}bool} $= 1$). We use the five features \texttt{prop\textunderscore{}starrating}, \texttt{prop\textunderscore{}review\textunderscore{}score}, \texttt{prop\textunderscore{}location\textunderscore{}score1}, \texttt{price\textunderscore{}usd}, and \texttt{promotion\textunderscore{}flag}, which we call \emph{star rating}, \emph{review score}, \emph{location score}, \emph{price}, and \emph{on promotion}, respectively.

\xhdr{\textsc{sushi}} Raw data downloaded from \url{http://www.kamishima.net/sushi/} as \texttt{sushi3-2016.zip}. We use the file \texttt{sushi3b.5000.10.order} and treat the top-ranked sushi as the selection from the set of 10 options. We use the features labeled \texttt{style}, \texttt{major group}, \texttt{the heaviness/oiliness in taste}, \texttt{how frequently the user eats the SUSHI}, \texttt{normalized price}, \texttt{how frequently the SUSHI is sold in sushi shop}. We call these features \emph{is maki}, \emph{is seafood}, \emph{oiliness}, \emph{popularity}, \emph{price}, and \emph{availability}. 
Additional details about these features are described by \citeauthor{kamishima2003nantonac}~\cite{kamishima2003nantonac}.

\xhdr{\textsc{car-a}, \textsc{car-b}} Raw data downloaded from \url{http://users.cecs.anu.edu.au/~u4940058/CarPreferences.html}. We call the first and second experiments \textsc{car-a} and \textsc{car-b}, respectively. In \textsc{car-a}, the four features are \emph{is SUV}, \emph{is manual}, \emph{engine capacity}, and \emph{is hybrid}. In \textsc{car-b}, the seven features are \emph{is sedan}, \emph{is SUV}, \emph{is hatchback}, \emph{is manual}, \emph{engine capacity}, \emph{is hybrid}, and \emph{is all-wheel-drive}.

\xhdr{\textsc{car-alt}} Raw data downloaded from \url{http://qed.econ.queensu.ca/jae/2000-v15.5/mcfadden-train/}. The 21 features are \emph{price divided by ln(income)}, \emph{range}, \emph{acceleration}, \emph{top speed}, \emph{pollution}, \emph{size}, \emph{"big enough"}, \emph{luggage space}, \emph{operating cost}, \emph{station availability}, \emph{SUV}, \emph{sports car}, \emph{station wagon}, \emph{truck}, \emph{van}, \emph{EV}, \emph{commute $< 5 \times$ EV}, \emph{college $\times$ EV}, \emph{CNG}, \emph{methanol}, \emph{college $\times$ methanol}. 
\citeauthor{mcfadden2000mixed}~\cite{mcfadden2000mixed} describe these features in detail and \citeauthor{brownstone1996transactions}~\cite{brownstone1996transactions} describe how the survey was conducted.

\subsection{Network Datasets}
See \Cref{sec:network_datasets} for details on how we turn timestamped edge data into triadic closure choices and the features we use. Those steps are shared across all of the following datasets.

\xhdr{\textsc{synthetic-mnl}} We used $\theta = [2, 1, 3, 1, 3, 5]$ and Poisson rate 5 to generate the network (see main text for details).

\xhdr{\textsc{synthetic-lcl}} We used $\theta = [2, 1, 3, 1, 3, 5]$,
\[A = \begin{bmatrix}
0 & 0 & 0 & 0 & 0 & 100\\
0 & 0 & 5 & 0 & 0 & 0\\
0 & -5 & 0 & 0 & 0 & 0\\
-5 & 0 & 0 & 0 & 0 & 0\\
0 & 0 & 0 & 0 & 0 & 0\\
0 & 0 & 5 & 0 & 0 & 0
\end{bmatrix},\]
 and Poisson rate 5 to generate the network (see main text for details).

\xhdr{\textsc{email-enron}} Raw data downloaded from \url{http://www.cs.cornell.edu/~arb/data/pvc-email-Enron/}. 

\xhdr{\textsc{email-eu}} Raw data downloaded from \url{https://snap.stanford.edu/data/email-Eu-core-temporal.html}. 

\xhdr{\textsc{email-w3c}} Raw data downloaded from \url{http://www.cs.cornell.edu/~arb/data/pvc-email-W3C/}. 

\xhdr{\textsc{sms-a}, \textsc{sms-b}, \textsc{sms-c}} Raw data downloaded from \url{https://www.pnas.org/content/107/44/18803/tab-figures-data}. We call ``Dataset S1,'' ``Dataset S2,'' and ``Dataset S3'' \textsc{sms-a}, \textsc{sms-b}, and \textsc{sms-c}, respectively.

\xhdr{\textsc{wiki-talk}} Raw data downloaded from \url{https://snap.stanford.edu/data/wiki-talk-temporal.html}. We reduce its size by only considering edges in 2004.

\xhdr{\textsc{reddit-hyperlink}} Raw data downloaded from \url{https://snap.stanford.edu/data/soc-RedditHyperlinks.html}. We consider only edges before 2015.

\xhdr{\textsc{bitcoin-alpha}} Raw data downloaded from \url{https://snap.stanford.edu/data/soc-sign-bitcoin-alpha.html}.

\xhdr{\textsc{bitcoin-otc}} Raw data downloaded from \url{https://snap.stanford.edu/data/soc-sign-bitcoin-otc.html}.

\xhdr{\textsc{facebook-wall}} Raw data downloaded from \url{http://konect.cc/networks/facebook-wosn-wall/}. 

\xhdr{\textsc{mathoverflow}} Raw data downloaded from \url{https://snap.stanford.edu/data/sx-mathoverflow.html}.

\xhdr{\textsc{college-msg}} Raw data downloaded from \url{https://snap.stanford.edu/data/CollegeMsg.html}.

\section{Empirical Analysis Appendix}\label{sec:empirical_appendix}

\subsection{LCL Identifiability in Datasets}\label{sec:dataset_lcl_identify}

We used the characterization in \Cref{thm:lcl_identify} to determine whether an LCL is uniquely identifiable from each of the datasets. We computed 
\begin{equation}\label{eq:dim_span}
	\dim\left(\Span \left\{
		\begin{bmatrix}
		x_C\\
		1
	\end{bmatrix} \otimes (x_i- x_C)
	\mid C \in \mathcal{C_D}, i \in C \right\}\right)
\end{equation}
for each dataset using \texttt{NumPy}'s \texttt{linalg.matrix\_rank}~\cite{oliphant2006guide,van2011numpy}. If the matrix formed by these vectors is full-rank ($d^2 + d$), then an LCL is identifiable by \Cref{thm:lcl_identify}. \Cref{tbl:identifiable} shows the LCL is identifiable in all network datasets and in \textsc{district-smart}, \textsc{expedia}, and \textsc{sushi}. Additionally, we checked the necessary condition of \Cref{prop:lcl_non_identifiable} ($d+1$ choice sets with affinely independent mean features) and found it to be satisfied in all datasets except \textsc{car-b}.

\begin{table}[t]
\caption{LCL identifiability in real-world datasets}
\label{tbl:identifiable}
\begin{tabular}{lrr}
\toprule
\textbf{Dataset} & \textbf{\Cref{thm:lcl_identify}}$^*$ & \textbf{\Cref{prop:lcl_non_identifiable}$^\dagger$}\\
\midrule
\textsc{district} & 368/756 & \textbf{28/28}\\
\textsc{district-smart} & \textbf{42/42} & \textbf{7/7}\\
\textsc{expedia} & \textbf{30/30} & \textbf{6/6}\\
\textsc{sushi} & \textbf{42/42} & \textbf{7/7}\\
\textsc{car-a} & 16/20 & \textbf{5/5}\\
\textsc{car-b} & 34/56 & 7/8\\
\textsc{car-alt} & 449/462 & \textbf{22/22}\\
\midrule
\textsc{synthetic-mnl} & \textbf{42/42} & \textbf{7/7}\\
\textsc{synthetic-lcl} & \textbf{42/42} & \textbf{7/7}\\
\textsc{wiki-talk} & \textbf{42/42} & \textbf{7/7}\\
\textsc{reddit-hyperlink} & \textbf{42/42} & \textbf{7/7}\\
\textsc{bitcoin-alpha} & \textbf{42/42} & \textbf{7/7}\\
\textsc{bitcoin-otc} & \textbf{42/42} & \textbf{7/7}\\
\textsc{sms-a} & \textbf{42/42} & \textbf{7/7}\\
\textsc{sms-b} & \textbf{42/42} & \textbf{7/7}\\
\textsc{sms-c} & \textbf{42/42} & \textbf{7/7}\\
\textsc{email-enron} & \textbf{42/42} & \textbf{7/7}\\
\textsc{email-eu} & \textbf{42/42} & \textbf{7/7}\\
\textsc{email-w3c} & \textbf{42/42} & \textbf{7/7}\\
\textsc{facebook-wall} & \textbf{42/42} & \textbf{7/7}\\
\textsc{college-msg} & \textbf{42/42} & \textbf{7/7}\\
\textsc{mathoverflow} & \textbf{42/42} & \textbf{7/7}\\
\midrule
\multicolumn{3}{l}{$^*$\footnotesize{Value of $\dim$ in \eqref{eq:dim_span} / $d^2 + d$}; bold if LCL is identifiable}\\
\multicolumn{3}{l}{$^\dagger$\footnotesize{Max \# of affinely ind.~choice sets / $d + 1$}; bold if necessary condition is satisfied}
\end{tabular}
\end{table}

\subsection{EM vs.~SGD Comparison}
\Cref{tbl:em_comparison} compares the NLL of the DLCL model learned by optimizing the likelihood with Adam and with the EM algorithm. 
For Adam, we used the hyperparameters found in the grid search specified in \Cref{sec:network_datasets}. For the EM algorithm, we iterated until the norm of the gradient of the NLL was less than $10^{-6}$ (with a 1-hour timeout, as for Adam). To optimize the convex function in the EM algorithm, we used \texttt{amsgrad} Adam with no weight decay. The number of iterations and the learning rate were selected by performing a grid search for each dataset over the possible iteration counts $5, 10, 50, 100$ and the learning rates $0.0005, 0.001, 0.005, 0.01, 0.05, 0.1$. The EM algorithm performs better on most of the datasets. Note that in \Cref{tbl:nll}, \Cref{tbl:prediction}, and \Cref{tbl:nll_full} we use the parameters learned by Adam.

\begin{table}[t]
\caption{Adam vs.~ EM algorithm for DLCL}
\label{tbl:em_comparison}
\begin{tabular}{lrr}
\toprule
\textbf{Dataset} & \textbf{Adam NLL} & \textbf{EM NLL}\\
\midrule
\textsc{district} & 3206 & \textbf{3041}\\
\textsc{district-smart} & 3303 & \textbf{3144}\\
\textsc{expedia} & 837569 & \textbf{805055}\\
\textsc{sushi} & 9764 & \textbf{9709}\\
\textsc{car-a} & 1692 & \textbf{1684}\\
\textsc{car-b} & 1284 & \textbf{1246}\\
\textsc{car-alt} & 7011 & \textbf{6369}\\
\midrule
\textsc{synthetic-mnl} & 210504 & \textbf{210462}\\
\textsc{synthetic-lcl} & 137937 & \textbf{137283}\\
\textsc{wiki-talk} & \textbf{95134} & 95985\\
\textsc{reddit-hyperlink} & \textbf{132473} & 132737\\
\textsc{bitcoin-alpha} & 18877 & \textbf{18401}\\
\textsc{bitcoin-otc} & 25348 & \textbf{25277}\\
\textsc{sms-a} & 8154 & \textbf{8112}\\
\textsc{sms-b} & \textbf{12975} & 12992\\
\textsc{sms-c} & 4871 & \textbf{4859}\\
\textsc{email-enron} & 69254 & \textbf{69249}\\
\textsc{email-eu} & 51431 & \textbf{51247}\\
\textsc{email-w3c} & 9758 & \textbf{9664}\\
\textsc{facebook-wall} & 116328 & \textbf{116098}\\
\textsc{college-msg} & 13712 & \textbf{13620}\\
\textsc{mathoverflow} & \textbf{435932} & 455103\\
\bottomrule
\end{tabular}
\end{table}

\begin{table*}[b]
\caption{Full version of negative log-likelihood table}
\label{tbl:nll_full}
\begin{tabular}{lrrrrrrrr}
\toprule
&\multirow{2}{*}{\textbf{MNL}} & \multirow{2}{*}{\textbf{LCL}} & \multicolumn{2}{c}{\textbf{Likelihood-ratio test}} & \multirow{2}{*}{\textbf{Mixed logit}} & \multirow{2}{*}{\textbf{DLCL}} & \multicolumn{2}{c}{\textbf{Likelihood-ratio test}}\\
\cmidrule{4-5}\cmidrule{8-9}
& & &$2(\ell_\text{MNL} - \ell_\text{LCL})$ & $p$-value & & & $2(\ell_\text{ML} - \ell_\text{DLCL})$ & $p$-value\\
\midrule
\textsc{district} & 3313 & \textbf{3130}\phantom{$^*$} & $367$ & $1.0$ & 3258 & 3206\phantom{$^{\dagger}$} & $105$ & $1.0$\\
\textsc{district-smart} & 3426 & \textbf{3278}$^*$ & $297$ & $< 10^{-16}$ & 3351 & 3303$^{\dagger}$ & $98$ & $1.2 \times 10^{-7}$\\
\textsc{expedia} & 839505 & 837649$^*$ & $3712$ & $< 10^{-16}$ & 839055 & \textbf{837569}$^{\dagger}$ & $2973$ & $< 10^{-16}$\\
\textsc{sushi} & 9821 & 9773$^*$ & $97$ & $1.6 \times 10^{-7}$ & 9793 & \textbf{9764}\phantom{$^{\dagger}$} & $59$ & $0.0089$\\
\textsc{car-a} & 1702 & 1694\phantom{$^*$} & $16$ & $0.44$ & 1696 & \textbf{1692}\phantom{$^{\dagger}$} & $9$ & $0.89$\\
\textsc{car-b} & 1305 & 1295\phantom{$^*$} & $20$ & $1.0$ & 1297 & \textbf{1284}\phantom{$^{\dagger}$} & $27$ & $1.0$\\
\textsc{car-alt} & 7393 & \textbf{6733}$^*$ & $1320$ & $< 10^{-16}$ & 7301 & 7011$^{\dagger}$ & $581$ & $7.6 \times 10^{-6}$\\
\midrule
\textsc{synthetic-mnl} & \textbf{210473} & 210486\phantom{$^*$} & $-25$ & $1.0$ & 210503 & 210504\phantom{$^{\dagger}$} & $0$ & $1.0$\\
\textsc{synthetic-lcl} & 140279 & \textbf{137232}$^*$ & $6095$ & $< 10^{-16}$ & 139539 & 137937$^{\dagger}$ & $3205$ & $< 10^{-16}$\\
\textsc{wiki-talk} & 99608 & 97748$^*$ & $3721$ & $< 10^{-16}$ & 95761 & \textbf{95134}$^{\dagger}$ & $1254$ & $< 10^{-16}$\\
\textsc{reddit-hyperlink} & 135108 & 132880$^*$ & $4457$ & $< 10^{-16}$ & 133766 & \textbf{132473}$^{\dagger}$ & $2587$ & $< 10^{-16}$\\
\textsc{bitcoin-alpha} & 19675 & 19190$^*$ & $971$ & $< 10^{-16}$ & 19093 & \textbf{18877}$^{\dagger}$ & $434$ & $< 10^{-16}$\\
\textsc{bitcoin-otc} & 26968 & 26101$^*$ & $1736$ & $< 10^{-16}$ & 25768 & \textbf{25348}$^{\dagger}$ & $842$ & $< 10^{-16}$\\
\textsc{sms-a} & 8252 & \textbf{8056}$^*$ & $393$ & $< 10^{-16}$ & 8239 & 8154$^{\dagger}$ & $171$ & $< 10^{-16}$\\
\textsc{sms-b} & 13153 & \textbf{12823}$^*$ & $660$ & $< 10^{-16}$ & 13147 & 12975$^{\dagger}$ & $345$ & $< 10^{-16}$\\
\textsc{sms-c} & 4988 & 4880$^*$ & $216$ & $< 10^{-16}$ & 4928 & \textbf{4871}$^{\dagger}$ & $116$ & $2.6 \times 10^{-10}$\\
\textsc{email-enron} & 73015 & 70061$^*$ & $5909$ & $< 10^{-16}$ & 71450 & \textbf{69254}$^{\dagger}$ & $4393$ & $< 10^{-16}$\\
\textsc{email-eu} & 53025 & 51822$^*$ & $2407$ & $< 10^{-16}$ & 51988 & \textbf{51431}$^{\dagger}$ & $1115$ & $< 10^{-16}$\\
\textsc{email-w3c} & 11012 & 10677$^*$ & $670$ & $< 10^{-16}$ & 9898 & \textbf{9758}$^{\dagger}$ & $280$ & $< 10^{-16}$\\
\textsc{facebook-wall} & 118208 & \textbf{116062}$^*$ & $4293$ & $< 10^{-16}$ & 117210 & 116328$^{\dagger}$ & $1766$ & $< 10^{-16}$\\
\textsc{college-msg} & 14575 & 14120$^*$ & $911$ & $< 10^{-16}$ & 13849 & \textbf{13712}$^{\dagger}$ & $276$ & $< 10^{-16}$\\
\textsc{mathoverflow} & 500537 & 479999$^*$ & $41077$ & $< 10^{-16}$ & 440482 & \textbf{435932}$^{\dagger}$ & $9101$ & $< 10^{-16}$\\
\midrule
\multicolumn{5}{l}{$^*$\footnotesize{Significant likelihood-ratio test vs.~MNL ($p<0.001$)}}\\
\multicolumn{5}{l}{$^{\dagger}$\footnotesize{Significant likelihood-ratio test vs.~mixed logit ($p< 0.001$)}}\\
\end{tabular}
\end{table*}

\end{document}